\documentclass[11pt]{article}
\usepackage[a4paper,margin=25mm]{geometry}
\usepackage{adjustbox}
\usepackage[T1]{fontenc}
\usepackage[utf8]{inputenc}
\usepackage{lmodern}
\usepackage{microtype}
\usepackage{amsmath,amssymb,amsfonts,mathtools,bm}
\usepackage{amsthm}
\usepackage{graphicx}
\usepackage{booktabs}
\usepackage{tabularx}
\usepackage{multirow}
\usepackage{enumitem}
\usepackage{algorithm}
\usepackage{algpseudocode}
\usepackage[authoryear]{natbib}

\usepackage{xcolor}
\definecolor{RevisionRed}{RGB}{190,0,0}
\newcommand{\RevisionBegin}{\begingroup}
\newcommand{\RevisionEnd}{\endgroup}
\usepackage{url}
\usepackage[colorlinks=true,citecolor=blue,linkcolor=red,urlcolor=blue]{hyperref}
\usepackage[nameinlink,capitalize]{cleveref}
\usepackage{needspace}
\BeforeBeginEnvironment{tabular}{\begin{adjustbox}{max width=\linewidth}}
\AfterEndEnvironment{tabular}{\end{adjustbox}}

\newtheorem{theorem}{Theorem}
\newtheorem{assumption}{Assumption}
\newtheorem{proposition}{Proposition}
\newtheorem{lemma}{Lemma}
\newtheorem{corollary}{Corollary}
\theoremstyle{definition}
\newtheorem{definition}{Definition}
\newtheorem{example}{Example}
\theoremstyle{remark}
\newtheorem{remark}{Remark}
\crefname{definition}{Definition}{Definitions}
\Crefname{definition}{Definition}{Definitions}
\crefname{example}{Example}{Examples}
\Crefname{example}{Example}{Examples}
\crefname{theorem}{Theorem}{Theorems}
\crefname{assumption}{Assumption}{Assumptions}
\crefname{proposition}{Proposition}{Propositions}
\crefname{lemma}{Lemma}{Lemmas}
\crefname{corollary}{Corollary}{Corollaries}
\crefname{remark}{Remark}{Remarks}
\Crefname{theorem}{Theorem}{Theorems}
\Crefname{assumption}{Assumption}{Assumptions}
\Crefname{proposition}{Proposition}{Propositions}
\Crefname{lemma}{Lemma}{Lemmas}
\Crefname{corollary}{Corollary}{Corollaries}

\newcommand{\E}{\mathbb E}

\newcommand{\R}{\mathbb R}
\newcommand{\one}{\mathbf 1}
\newcommand{\calB}{\mathcal B}
\newcommand{\calD}{\mathcal D}
\newcommand{\calG}{\mathcal G}
\newcommand{\calF}{\mathcal F}

\newcommand{\calW}{\mathcal W}
\newcommand{\ECE}{\operatorname{ECE}}
\newcommand{\logit}{\operatorname{logit}}

\newcommand{\Var}{\operatorname{Var}}
\newcommand{\Cov}{\operatorname{Cov}}

\newcommand{\argmin}{\operatorname*{arg\,min}}
\newcommand{\argmax}{\operatorname*{arg\,max}}
\newcommand{\op}{\mathrm{op}}
\newcommand{\nc}{\mathrm{NC}}
\newcommand{\NCECE}{NC-ECE}
\newcommand{\NCECEname}{Noise Corrected ECE}
\newcommand{\ch}{\mathrm{ch}}
\newcommand{\ev}{\mathrm{ev}}

\newcommand{\mom}{\mathrm{mom}}
\newcommand{\GMM}{\mathrm{GMM}}

\newcommand{\indep}{\mathrel{\perp\!\!\!\perp}}

\newcommand{\wt}{\widetilde}
\newcommand{\wh}{\widehat}

\newcommand{\abs}[1]{\left|#1\right|}

\hypersetup{pdftitle={Noisy-Label Calibration with Multiple Unreliable Labels},pdfauthor={Keisuke Hanada, Futoshi Futami}}

\newcommand{\calX}{\mathcal X}
\newcommand{\calY}{\mathcal Y}
\newcommand{\calC}{\mathcal C}
\newcommand{\calH}{\mathcal H}
\newcommand{\calJ}{\mathcal J}
\newcommand{\Bn}{\widehat B}
\newcommand{\dch}{d_\ch}

\newcommand{\IV}{\mathrm{IV}}
\newcommand{\rec}{\mathrm{rec}}
\newcommand{\val}{\mathrm{val}}
\newcommand{\Ber}{\operatorname{Bernoulli}}
\newcommand{\essinf}{\operatorname*{ess\,inf}}
\newcommand{\esssup}{\operatorname*{ess\,sup}}
\newcommand{\Rad}{\mathfrak R}
\allowdisplaybreaks[0]

\title{Noisy-Label Calibration with Multiple Unreliable Labels}
\author{Keisuke Hanada$^{1,3}$, Futoshi Futami$^{2,3,4}$ \\[1ex]
        $^{1}$Department of Biostatistics, Wakayama Medical University, Wakayama, Japan \\
        $^{2}$Graduate School of Engineering Science, The University of Osaka, Osaka, Japan \\
        $^{3}$RIKEN AIP, Tokyo, Japan \\
        $^{4}$Department of Complexity Science and Engineering, The University of Tokyo, Tokyo, Japan}
\date{}

\begin{document}
\maketitle
\begin{abstract}
\RevisionBegin
\emergencystretch=6pt
\RevisionEnd
In medical annotation and crowdsourcing, multiple annotators assign labels to each example, but these labels may be incorrect. Treating such noisy labels as ground truth can bias calibration evaluation. To address this problem, we propose methods for evaluating and recalibrating a fixed binary predictor under the underlying clean-label distribution, using only noisy annotations. First, given annotator-specific noise probabilities, we propose \NCECEname{} (\NCECE{}), a binned estimator of expected calibration error (ECE) under class-conditional noise. We derive non-asymptotic bounds on the estimator's bias. Then, we propose practical algorithms for estimating unknown noise probabilities under conditional independence and label-orientation assumptions. We also develop loss-based and histogram post-hoc recalibration with clean-risk and calibration guarantees, respectively. Experiments on semi-synthetic benchmarks show that the proposed ECE estimator closely tracks its clean-label oracle counterpart and that the recalibration methods substantially improve latent-label calibration, whereas uncorrected recalibration against noisy annotations can degrade it.

\end{abstract}
\section{Introduction}
\label{v03:introduction}

Using probabilistic predictions for decisions requires more than accurate
class predictions: the probabilities should also be calibrated.
Calibration means that predicted probabilities agree with outcome
frequencies among examples assigned those probabilities.
For example, among examples assigned an eighty percent probability of
a positive outcome, about eighty percent should be positive
\citep{Guo2017,Vaicenavicius2019}.
Calibration evaluation and post-hoc recalibration usually treat held-out
labels as ground truth~\citep{Guo2017,Fujisawa2025}.
In medical annotation and crowdsourcing, however, multiple annotators
label each example, and these labels may contain errors
\citep{raykar_learning_2010,tanno2019learning}.
Medical experts may disagree in their annotations
\citep{gulshan2016development,elmore2015diagnostic},
and majority voting does not necessarily eliminate annotation errors \citep{krause2018grader}.
We study how to evaluate and improve a fixed binary predictor's
calibration under the underlying clean-label distribution when only
noisy annotations are available.

Noisy-label evaluation confounds predictor miscalibration with annotator
response tendencies~\citep{dawid1979maximum,penso2024confidence}.
For example, excess positive responses can reflect either underpredicted
probabilities or an annotator's tendency to report positives.
With a fixed number of annotators, more evaluation examples generally
leave this bias intact; averaging or majority vote need not remove it either.
Recalibration toward annotation agreement can even worsen true-label
calibration.
Reliable uncertainty assessment is important in medical decision support,
where expert judgments may disagree~\citep{pmlr-v130-mimori21a}.
Noisy evaluation labels \mbox{complicate predictive reliability assessment}.

Loss correction for learning with noisy labels provides our starting point.
Known noise probabilities can yield unbiased estimators of clean-label
risk~\citep{natarajan2013learning,patrini2017making,JMLR:v18:16-315}.
We apply this principle to calibration evaluation in
Section~\ref{v03:estimation}.
Under annotator-specific class-conditional noise, inverting each known
channel yields a corrected label whose conditional expectation given
the prediction matches that of the true label.
Subtracting the prediction gives a corrected residual with the desired
conditional mean (Lemma~\ref{lem:corrected-main}).
This gives unbiased estimates of clean-label bin residual moments,
whose absolute values are summed to form a noise-corrected binned ECE estimator
(Definition~\ref{def:corrected-binning-ece}).

The required noise probabilities, however, are usually unknown.
Response frequencies mix annotator errors with true-class proportions
and need not identify either~\citep{hui1980estimating}.
Multiple annotations of the same example supply additional information
\citep{dawid1979maximum,allman2009identifiability}.
Under conditional independence and label-orientation assumptions,
we use their joint distribution, supplemented where necessary by
score-defined groups.
Section~\ref{v03:identification} characterizes identification conditions
and constructs channel estimators accordingly, using two annotators
with score-defined groups (Theorem~\ref{v03:k2}) or three or more informative annotators.
We then account for their estimation errors in ECE evaluation rather
than treating estimated \mbox{noise probabilities as exact}.

Our analysis connects noise-probability estimation to calibration
evaluation under the clean-label distribution, separating bin
approximation, sampling variation, and channel-estimation errors
(Theorem~\ref{v03:bound}).
The same corrected labels support post-hoc recalibration in
Section~\ref{v03:recalibration}: corrected-loss fitting provides a
clean-risk guarantee,
while histogram fitting provides a calibration guarantee
(Proposition~\ref{v03:histogram}).
Both evaluation and recalibration thus address the underlying clean-label
distribution using noisy annotations alone.

Our contributions are threefold.
\begin{enumerate}
\item \textbf{Calibration evaluation.}
We propose a noise-corrected binned ECE estimator with separate binning, sampling, and
channel-estimation error bounds (Theorem~\ref{v03:bound}), including
sample reuse with fixed bins (Proposition~\ref{cor:reuse-main}).
We also optimize annotator weights to reduce residual variance
(Lemma~\ref{lem:iv-main}).

\item \textbf{Recalibration.}
Corrected-loss and projected histogram fitting provide clean-risk and calibration guarantees, respectively (Section~\ref{v03:recalibration}; Proposition~\ref{v03:histogram}).

\item \textbf{Empirical validation.}
\looseness=-1 
Controlled and semi-synthetic experiments assess estimation and recalibration (Section~\ref{v03:experiments}). Human annotations show how label choice can reverse recalibration assessments.
\end{enumerate}

\paragraph{Related work.}
Calibration evaluation and binning bias have been studied extensively
\citep{Vaicenavicius2019,roelofs2022mitigating,
futami2024information}. Weighted test functions provide a common view of
calibration criteria \citep{blasiok2023unifying,gopalan2022lowdegree}.
Unbiased class-conditional loss correction is also established
\citep{natarajan2013learning}. Kiryo et al.\ \citeyearpar{kiryo2025pu}
estimate ECE from positive and unlabeled data, and
\citet{futami2026weak} develop decontamination-based multicalibration
estimators and post-processing for weak supervision. We build on this
correction principle. Our focus is the statistical link from unknown,
annotator-specific rates in repeated annotations to ECE evaluation and
recalibration, including finite-sample error propagation and sample reuse.

For identification, our two-group construction instantiates the
Hui--Walter setting of two imperfect tests in populations with different
prevalences \citep{hui1980estimating}. Here the populations are observable
score groups. The identification structure is classical. Our analysis
connects quantitative channel-estimation guarantees in this setting to
latent ECE. With additional annotators, we use the classical
Dawid--Skene product-mixture model \citep{dawid1979maximum,
allman2009identifiability}. Joint noisy-label classifier training
\citep{tanno2019learning} estimates a predictor together with the annotation
model. Our analysis evaluates and recalibrates a fixed predictor.
Disagreement-based approaches model the distribution of annotator judgments
\citep{pmlr-v130-mimori21a,riera2026multirater}. Our target is calibration
to one latent binary outcome. The distinction between correcting known channels and identifying unknown
ones determines the assumptions and scope of our results.

\section{Preliminaries}
\label{v03:setup}
\paragraph{Notation.}
For a positive integer $d$, let $[d]\coloneqq\{1,\ldots,d\}$.
Capital letters denote random variables and lower-case letters their
realizations. We write $P_V$ for the distribution of $V$ and $P_{V\mid W}$
for its conditional distribution given $W$. Probabilities and expectations are denoted by $P$ and $\E$. The indicator $\one(E)$ equals one on event
$E$ and zero otherwise. Conditional-expectation equalities hold almost
surely. \mbox{Further notation is in Appendix~\ref{v03:notation}.}

\subsection{Calibration in binary classification}
\label{v03:ordinary-calibration}
Let $\calX$ and $\calY\coloneqq\{0,1\}$ be the input and label spaces, and let
$(X,Y)\sim P_{X,Y}$. Fix a measurable predictor $f:\calX\to[0,1]$ and
write $Z\coloneqq f(X)$. The score $f(x)$ is intended to estimate $\E[Y\mid X=x]$. Calibration compares a predicted probability with the positive-label frequency among inputs assigned that probability.

\begin{definition}[Expected calibration error]
\label{def:ece}
The predictor $f$ is \emph{perfectly calibrated} if $\E[Y\mid Z]=Z$
almost surely. Its \emph{expected calibration error} (ECE) is
\begin{equation}
 \ECE(f)\coloneqq\E\left[\left|\E[Y\mid Z]-Z\right|\right].
 \label{v03:population}
\end{equation}
\end{definition}

To estimate ECE from data, binning groups examples with similar predicted
probabilities \citep{Guo2017}.
Let $S=\{(X_i,Y_i)\}_{i=1}^n$ be an independent and identically
distributed (i.i.d.) sample from $P_{X,Y}$, independent of the training
of $f$, and put $Z_i\coloneqq f(X_i)$.

\begin{definition}[Binning ECE]
\label{def:binning-ece}
Let $\calB\coloneqq\{B_1,\ldots,B_M\}$ partition $[0,1]$ into $M\ge1$
measurable bins. For $\ell\in[M]$, set
$I_\ell\coloneqq\{i\in[n]:Z_i\in B_\ell\}$ and $n_\ell\coloneqq|I_\ell|$.
For $n_\ell>0$, define
$\bar Y_\ell\coloneqq n_\ell^{-1}\sum_{i\in I_\ell}Y_i$ and
$\bar Z_\ell\coloneqq n_\ell^{-1}\sum_{i\in I_\ell}Z_i$;
set both means to zero otherwise.
The \emph{Binning ECE} estimator is
\begin{align}
 \wh\ECE_\calB(f,S)
 &\coloneqq\textstyle\sum_{\ell=1}^M(n_\ell/n)|\bar Y_\ell-\bar Z_\ell|\notag\\
 &=\textstyle\sum_{\ell=1}^M\left|
       n^{-1}\textstyle\sum_{i=1}^n(Y_i-Z_i)\one(Z_i\in B_\ell)\right|.
 \label{v03:clean-ece}
\end{align}
\end{definition}
Binning ECE compares each bin's positive-label frequency with its average prediction. Sampling introduces error in these averages, while binning
can conceal opposite residuals within a bin. Both errors must be accounted
for when estimating $\ECE(f)$ \citep{roelofs2022mitigating,futami2024information}.

\subsection{Multi-annotator setting}
\label{v03:annotation-model}
In medical annotation and crowdsourcing, multiple annotators may supply noisy labels for each input when clean evaluation labels are unavailable \citep{raykar_learning_2010,tanno2019learning}.
Let $K\ge1$ be the number of annotators and $A_j\in\calY$ the label
supplied by annotator $j\in[K]$. Write $A_{1:K}\coloneqq(A_1,\ldots,A_K)$.
The annotation model $P_{A_{1:K}\mid X,Y}$ extends $P_{X,Y}$ to the joint distribution of inputs, true labels, and annotations. We impose the following condition given $(Y,Z)$.

Assume $0<P(Y=1)<1$. For each $j\in[K]$, define
$a_j\coloneqq P(A_j=1\mid Y=0)$, $b_j\coloneqq P(A_j=1\mid Y=1)$, and
$\delta_j\coloneqq b_j-a_j$.
The pair $(a_j,b_j)$ is the annotator's binary \emph{noise channel}:
$a_j$ is its false-positive probability and $1-b_j$ its false-negative
probability. The \emph{channel gap} is $\delta_j$; the channel is
\emph{informative} if $\delta_j\ne0$.
\begin{assumption}[Class-conditional annotation noise]
\label{v03:marginal}
For every $j\in[K]$, $\E[A_j\mid Y,Z]=a_j+\delta_jY$.
\end{assumption}
Since $A_j$ is binary, this excludes score-dependent errors: its error
probabilities do not vary with the score once the true label is fixed. Write
$m(z)\coloneqq\E[Y\mid Z=z]$ for a measurable version
$m:[0,1]\to[0,1]$ of the conditional expectation. Then
\begin{equation}
 \E[A_j\mid Z]=a_j+\delta_jm(Z),\qquad j\in[K].
 \label{v03:annotation-mean}
\end{equation}
Let $A_{ij}$ denote annotator $j$'s label for example $i$.
Evaluating ECE directly from annotator $j$'s labels amounts to replacing
$Y_i$ by $A_{ij}$ in Eq.~\eqref{v03:clean-ece}. The label-side conditional mean is
then $a_j+\delta_jm(Z)$ rather than $m(Z)$.
Annotation noise therefore generally biases estimation of ECE with respect
to the true label. Increasing the evaluation sample size, averaging
annotations, or taking a majority vote does not in general remove this
noise-induced bias (Appendix~\ref{app:aggregation}).

\section{Calibration Estimation}
\label{v03:estimation}
We evaluate $\ECE(f)$ using known channel probabilities $a_j,b_j$, $j\in[K]$, or supplied estimates $\wh a_j,\wh b_j$. Section~\ref{v03:identification} gives methods for estimating these probabilities.

\label{sec:corrected-estimation}
\label{v03:estimator-samples}
For each $j\in[K]$ with known channel probabilities and $\delta_j\ne0$,
Eq.~\eqref{v03:annotation-mean} gives
\[
 \E[Y\mid Z]
 =(\E[A_j\mid Z]-a_j)/\delta_j
 =\E[(A_j-a_j)/\delta_j\mid Z].
\]
Applying this inverse transformation to an annotation therefore recovers
the true label's conditional expectation given the prediction.

\begin{definition}[Corrected labels and residuals]
\label{def:corrected-labels}
For known channels with $\delta_j\ne0$ for every $j\in[K]$, define
$U_j\coloneqq(A_j-a_j)/\delta_j$ and $\psi_j\coloneqq U_j-Z$.
We call $U_j$ the \emph{corrected label} and $\psi_j$ the
\emph{corrected residual} of annotator $j$.
Let $\Delta_K\coloneqq\{\alpha\in[0,1]^K:\sum_{j=1}^K\alpha_j=1\}$.
For fixed $\alpha\in\Delta_K$, define their convex combinations by
\begin{equation}
 U_\alpha\coloneqq\textstyle\sum_{j=1}^K\alpha_jU_j,
 \qquad \psi_\alpha\coloneqq\textstyle\sum_{j=1}^K\alpha_j\psi_j.
 \label{v03:corrected}
\end{equation}
Thus $\psi_\alpha=U_\alpha-Z$.
\end{definition}

\begin{lemma}[Corrected ECE representation]
\label{lem:corrected-main}
Under Assumption~\ref{v03:marginal}, the corrections in
Definition~\ref{def:corrected-labels} satisfy, for every fixed
$\alpha\in\Delta_K$,
\begin{equation}
 \ECE(f)=\E\bigl[\bigl|\E[\psi_\alpha\mid Z]\bigr|\bigr].
 \label{v03:identity}
\end{equation}
\end{lemma}
The proof is in Appendix~\ref{app:weighted-derivations}. The identity holds for every fixed $\alpha$. Under conditional independence, Section~\ref{sec:annotator-weights} derives weights minimizing the total variance of the bin-moment estimates.

\paragraph{Estimation from finite samples.}
We estimate the conditional residual in Eq.~\eqref{v03:identity} by bin averages, replacing unknown channel probabilities with estimates. If $f$ is learned, condition on its training data and randomness. Let $\calD_\ch$ and $\calD_\ev$ be independent i.i.d. samples of $O\coloneqq(Z,A_{1:K})\sim P_O$, with sizes $n_\ch$ and $n_\ev$. Neither sample contains true labels; both are independent of the training of $f$. 

\begin{definition}[\NCECEname{} (\NCECE{})]
\label{def:corrected-binning-ece}
For each $j\in[K]$, let $\wh a_j,\wh b_j$ be measurable estimators of
$a_j,b_j$ based on $\calD_\ch$, with $0\le\wh a_j<\wh b_j\le1$.
Let $\wh\alpha\in\Delta_K$ and the evaluation partition $\calB$ also
be chosen measurably from $\calD_\ch$, not $\calD_\ev$.
Set $\wh\delta_j\coloneqq\wh b_j-\wh a_j$.
For each example $i\in[n_\ev]$ and annotator $j\in[K]$, define
$\wh U_{ij}\coloneqq(A_{ij}-\wh a_j)/\wh\delta_j$ and
$\wh\psi_{ij}\coloneqq\wh U_{ij}-Z_i$.
Their weighted aggregates are
\begin{equation}
 \wh U_i\coloneqq\textstyle\sum_{j=1}^K\wh\alpha_j\wh U_{ij},
 \qquad \wh\psi_i\coloneqq\textstyle\sum_{j=1}^K\wh\alpha_j\wh\psi_{ij}.
 \label{eq:estimated-corrected-label-main}
\end{equation}
Then $\wh\psi_i=\wh U_i-Z_i$.
Define $I_\ell,n_\ell,\bar Z_\ell$ on $\calD_\ev$ as in
Definition~\ref{def:binning-ece}, and set
$\bar U_\ell\coloneqq n_\ell^{-1}\sum_{i\in I_\ell}\wh U_i$
when $n_\ell>0$, and zero otherwise. In a nonempty bin,
$\bar U_\ell-\bar Z_\ell$ is the sample mean of the corrected residuals
$\wh\psi_i$. 
We define \emph{\NCECEname{}} (\NCECE{}) by
\begin{equation}
 \wh\ECE^{\nc}_\calB(f,\calD_\ev)
 \coloneqq\textstyle\sum_{\ell=1}^M\frac{n_\ell}{n_\ev}
       |\bar U_\ell-\bar Z_\ell|.
 \label{v03:direct}
\end{equation}
\end{definition}
Write $\wh U(A_{1:K})$ for the fitted correction of a new annotation vector. Its residual is $\wh\psi(O)\coloneqq\wh U(A_{1:K})-Z$.

\paragraph{Error relative to ECE.}
The estimation error separates into bin approximation, evaluation-sample fluctuation, and channel-estimation error.

\begin{theorem}[ECE estimation with independent samples]
\label{v03:bound}
Suppose Assumption~\ref{v03:marginal} holds, and use the estimator in
Definition~\ref{def:corrected-binning-ece} with independent samples
$\calD_\ch,\calD_\ev$. Assume $r(z)\coloneqq m(z)-z$ has a version
satisfying $|r(z)-r(z')|\le L_r|z-z'|$ on the support of $P_Z$.
Let $h_\calB\coloneqq\max_{\ell\in[M]}\operatorname{diam}(B_\ell)$,
$\Bn\coloneqq\sum_{j=1}^K\wh\alpha_j/\wh\delta_j$, and define
\begin{equation}
 \textstyle\dch\coloneqq\sum_{j=1}^K\wh\alpha_j
 \max(|\wh a_j-a_j|,|\wh b_j-b_j|)/\wh\delta_j.
 \label{v03:channel-error}
\end{equation}
For every $\eta\in(0,1)$, conditional on $\calD_\ch$, with probability
at least $1-\eta$,
\begin{equation}
\begin{aligned}[b]
 &\abs{\wh\ECE^{\nc}_\calB(f,\calD_\ev)-\ECE(f)}\\
 &\le L_rh_\calB+
 \Bn\sqrt{2(M\log2+\log(2/\eta))/n_\ev}+\dch.
 \label{v03:bound-eq}
\end{aligned}
\end{equation}
\end{theorem}
The proof is in Appendix~\ref{v03:full-ece}.
\begin{remark}[Three sources of error]
\label{rem:ece-bound}
The terms represent binning, sampling, and channel-estimation errors.
Smoothness controls binning; the proof also gives a smoothness-free
fixed-bin bound. Because $\dch$ contains unknown parameters, the bound
alone does not yield a computable confidence interval.
\end{remark}

\paragraph{Convergence rates.}
Define the maximum channel estimation error across annotators by
\begin{equation}
 e_\ch\coloneqq\max_{1\le j\le K}
       \bigl(|\wh a_j-a_j|+|\wh b_j-b_j|\bigr).
 \label{eq:channel-estimation-error}
\end{equation}

\begin{corollary}[Convergence rates]
\label{cor:ece-rate-main}
Under Theorem~\ref{v03:bound}, keep $K$ fixed and assume
$\delta_j\ge\delta_{\min}>0$ for every $j\in[K]$.
Suppose $e_\ch=O_p(\varepsilon_{n_\ch})$, where
$\varepsilon_{n_\ch}>0$ is a deterministic sequence tending to zero.
For $M$ equal-width evaluation bins,
\begin{equation}
\begin{aligned}[b]
 &\abs{\wh\ECE^{\nc}_\calB(f,\calD_\ev)-\ECE(f)}\\
 &=O_p\!\left(M^{-1}+\sqrt{M/n_\ev}+\varepsilon_{n_\ch}\right).
 \label{v03:rate}
\end{aligned}
\end{equation}
Thus $M\asymp n_\ev^{1/3}$ gives
$O_p(n_\ev^{-1/3}+\varepsilon_{n_\ch})$.

\end{corollary}
The proof is in Appendix~\ref{app:ece-rates}. If
$e_\ch=O_p(n_\ch^{-1/2})$ and $M\asymp n_\ev^{1/3}$, the error is
$O_p(n_\ev^{-1/3}+n_\ch^{-1/2})$.
In particular, $n_\ch\asymp n_\ev\asymp n$ gives $O_p(n^{-1/3})$.
This matches the order of existing binning-based ECE estimation bounds with clean labels~\citep{futami2024information}.

\paragraph{Reusing the estimation sample.}
We next reuse one sample for channel estimation and ECE evaluation, with
$f$ and the evaluation bins fixed independently of that sample.
\begin{proposition}[ECE estimation with a reused sample]
\label{cor:reuse-main}
Suppose Assumption~\ref{v03:marginal} and the Lipschitz condition of
Theorem~\ref{v03:bound} hold. Fix $f,\calB$ independently of an i.i.d.
sample $\calD$ of size $n$. Let $\wh a_j,\wh b_j$, $j\in[K]$, be
measurable channel estimates fitted on $\calD$, with
$0\le\wh a_j<\wh b_j\le1$, and let $\wh\alpha\in\Delta_K$ also depend
on $\calD$. Compute the estimator in
Definition~\ref{def:corrected-binning-ece} on $\calD$ with these values.
For every $\eta\in(0,1)$, with probability at least $1-\eta$,
\begin{equation}
\begin{aligned}[b]
 &\abs{\wh\ECE^{\nc}_\calB(f,\calD)-\ECE(f)}\\
 &\le L_rh_\calB+
 \Bn\sqrt{2(M\log2+\log(6K/\eta))/n}+\dch,
 \label{v03:same-sample}
\end{aligned}
\end{equation}
where $\Bn,\dch$ are defined in Theorem~\ref{v03:bound} using estimates \mbox{fitted on $\calD$.}
\end{proposition}
The proof is in Appendix~\ref{v03:full-same-sample};
Appendix~\ref{app:cross-fitting} treats the \mbox{estimator based on cross-fitting}.

With fixed $K$, $\delta_j\ge\delta_{\min}>0$ for all $j\in[K]$, and
$e_\ch=O_p(\varepsilon_n)$, Proposition~\ref{cor:reuse-main} also gives
the rate in Corollary~\ref{cor:ece-rate-main}, with
$\calD_\ev,n_\ev,n_\ch$ replaced by $\calD,n,n$.

\section{Channel and Weight Estimation}
\label{v03:identification}
\subsection{Setting and assumptions}
\label{sec:channel-setting}
Section~\ref{v03:estimation} requires channel probabilities $a_j,b_j$.
Without true labels, response frequencies mix annotation errors and class
proportions; we identify these quantities using joint responses and score groups.

We retain Assumption~\ref{v03:marginal} and assume that annotations are
independent conditional on the true label and score, extending the class-conditional noise model proposed by \citet{natarajan2013learning} to multiple annotators.
\RevisionBegin
\Needspace{3\baselineskip}
\RevisionEnd
\begin{assumption}[
Conditionally independent annotations
]
\label{v03:product}
For every $s=(s_1,\ldots,s_K)\in\{0,1\}^K$,
\[
 P(A_{1:K}=s\mid Y,Z)
 =\prod_{j=1}^{K}P(A_j=s_j\mid Y,Z).
\]
\end{assumption}
Together, Assumptions~\ref{v03:marginal} and~\ref{v03:product} give independence
conditional on $Y$ and channels shared across score groups (Appendix \ref{app:assumptions}).

The observation distribution must determine the parameters
(identification). One annotator, or two without scores, are generally
insufficient (Appendix~\ref{app:nonidentification}).
\begin{remark}[Label orientation]
\label{rem:orientation-main}
Relabeling $Y^\star\coloneqq1-Y$ exchanges $a_j,b_j$ and complements
class proportions without changing the observations. To resolve this
exchange, we specify an orientation, such as $a_j<b_j$ for a designated
annotator. Orientation alone does not resolve this information shortage.
\end{remark}

\subsection{Two annotators and two score groups}
\label{sec:two-group-method}
Let $K=2$. Choose a threshold $t\in(0,1)$ in advance or on data independent
of $\calD_\ch$, and define $G\coloneqq\one(Z\ge t)$.
Here $G=0$ is the low-score group and $G=1$ the high-score group.
The group indicator $G$ is not the true label $Y$: either group may
contain both classes. For $g\in\{0,1\}$, assume $P(G=g)>0$ and let
$\mu_g\coloneqq P(Y=1\mid G=g)$ be its unknown positive-class proportion.
We jointly estimate $\theta\coloneqq(a_1,b_1,a_2,b_2,\mu_0,\mu_1)$.
The groups share channels; \mbox{their class proportions may differ.}

\paragraph{Estimation idea and observable moments.}
We recover channel probabilities from response-pair means and covariances
within each score group, together with differences between the group means.
For each $g\in\{0,1\}$, define
$q_{1,g}\coloneqq\E[A_1\mid G=g]$,
$q_{2,g}\coloneqq\E[A_2\mid G=g]$, and
$q_{12,g}\coloneqq\E[A_1A_2\mid G=g]$.
They are the probabilities that annotator 1, annotator 2, and both answer
positively, respectively. Averaging over the true label gives the following
three identities for each $g\in\{0,1\}$, hence six moment equations in total:
\begin{align}
 q_{1,g}&=a_1+\delta_1\mu_g,\quad q_{2,g}=a_2+\delta_2\mu_g,\notag\\
 q_{12,g}&=(1-\mu_g)a_1a_2+\mu_gb_1b_2,
 \qquad g\in\{0,1\}.
 \label{v03:moments}
\end{align}
Conditional independence gives the joint-positive terms and,
for $g\in\{0,1\}$,
\begin{align*}
 q_{1,1}-q_{1,0}&=\delta_1(\mu_1-\mu_0),\\
 q_{2,1}-q_{2,0}&=\delta_2(\mu_1-\mu_0),\\
 q_{12,g}-q_{1,g}q_{2,g}&=\delta_1\delta_2\mu_g(1-\mu_g).
\end{align*}
The last line is the within-group covariance. These relations yield a closed-form inverse; its lengthy expression is
given in Appendix~\ref{app:k2}.

\begin{theorem}[Two-group identification]
\label{v03:k2}
Suppose Assumptions~\ref{v03:marginal} and~\ref{v03:product} hold with
$K=2$, $0\le a_j<b_j\le1$ for $j\in[2]$, and
$0<\mu_0,\mu_1<1$ with $\mu_0\ne\mu_1$.
For the two positive-probability groups above, the six moments
\[
 \Gamma(\theta)\coloneqq
 (q_{1,0},q_{2,0},q_{12,0},q_{1,1},q_{2,1},q_{12,1})
\]
uniquely determine $\theta$.
\end{theorem}
Appendix~\ref{app:k2} gives the inverse and proof. The class proportions
are estimated, not supplied. Calibration of $f$ is unnecessary, but a
low/high score split alone does not ensure $\mu_0\ne\mu_1$.

\paragraph{Finite-sample estimation.}
On $\calD_\ch$, let $G_i\coloneqq\one(Z_i\ge t)$ and
$n_g\coloneqq\sum_{i=1}^{n_\ch}\one(G_i=g)$.
For $g\in\{0,1\}$ with $n_g>0$, estimate the three moments by
\begin{align}
 \wh q_{j,g}&\coloneqq n_g^{-1}\textstyle\sum_{i:G_i=g}A_{ij},
 \qquad j\in[2],\notag\\
 \wh q_{12,g}&\coloneqq n_g^{-1}\textstyle\sum_{i:G_i=g}A_{i1}A_{i2}.
 \label{eq:group-empirical-main}
\end{align}
Let $\wh\Gamma$ collect the six estimates in the order of $\Gamma(\theta)$.
The closed-form estimator is $\wh\theta_{\mom}\coloneqq H(\wh\Gamma)$,
where $H$ is the inverse in Theorem~\ref{v03:k2}, whenever it is feasible.

Sampling error can make the closed-form inverse infeasible. We therefore
also fit the same moment model by constrained generalized method of moments
(GMM; Appendices~\ref{v03:channel-experiments} and~\ref{app:gmm}). Under
separation and interiority, the feasible inverse and an exact GMM minimizer
satisfy $e_\ch=O_p(n_\ch^{-1/2})$ (Theorems~\ref{thm:k2-direct-finite}
and~\ref{thm:gmm-finite}), supplying the rate in
Corollary~\ref{cor:ece-rate-main}. Constraints do not resolve weak identification.

\subsection{Three or more annotators}
\label{sec:three-annotator-method}
With at least three informative annotators, the response vectors alone
suffice under Assumptions~\ref{v03:marginal} and~\ref{v03:product}:
no threshold or score-group separation is required.

\paragraph{Response distribution.}
Let $K\ge3$, $\mu\coloneqq P(Y=1)\in(0,1)$, and
$\theta\coloneqq(\mu,a_1,b_1,\ldots,a_K,b_K)$, with true value $\theta_0$.
Under Assumptions~\ref{v03:marginal} and~\ref{v03:product}, the responses
are independent conditional on $Y$. Marginalizing $Y$ therefore gives
the joint response distribution $P(A_{1:K}=s)=p_{\theta_0}(s)$, where
\begin{align}
 p_\theta(s)&\coloneqq(1-\mu)\prod_{j=1}^{K}a_j^{s_j}(1-a_j)^{1-s_j}\notag\\
 &\quad+\mu\prod_{j=1}^{K}b_j^{s_j}(1-b_j)^{1-s_j},
 \quad s\in\{0,1\}^K.
 \label{v03:ds}
\end{align}
Eq.~\eqref{v03:ds} is the binary Dawid--Skene model
\citep{dawid1979maximum}, a mixture of response distributions for the two classes.

\paragraph{Maximum-likelihood estimation.}
Using only responses, not predictions $Z=f(X)$, we jointly estimate
$\mu$ and $(a_j,b_j)_{j\in[K]}$ on a compact set
$\Theta\subset(0,1)^{2K+1}$ containing $\theta_0$ with the prescribed
orientation. On $\calD_\ch$, define the log likelihood and a measurable maximizer:
\begin{align}
 L_{n_\ch}(\theta)&\coloneqq n_\ch^{-1}\textstyle\sum_{i=1}^{n_\ch}
           \log p_\theta(A_{i,1:K}),\notag\\
 \wh\theta_{\mathrm{ML}}&\in\argmax_{\theta\in\Theta}L_{n_\ch}(\theta).
 \label{eq:ds-mle-main}
\end{align}
The channel coordinates give $\wh a_j,\wh b_j$ for the correction in
Section~\ref{v03:estimation}. Numerical likelihood maximization uses
expectation--maximization (EM; Appendix~\ref{v03:algorithms}).
A finite EM run need not reach a global maximizer~\citep{wu1983convergence}.

\paragraph{Why three annotations suffice.}
The third-order central moment, together with the three pairwise
covariances, determines $\mu$ up to its complement. The prescribed
orientation resolves this exchange; response means and covariances then
recover the channels. Proposition~\ref{prop:ds-main} gives the exact
inverse, including $\mu=1/2$ \citep{allman2009identifiability}.

For fixed $K$, Theorem~\ref{thm:kge3-regular-mle} gives
$O_p(n_\ch^{-1/2})$ error under its identification, interiority, and
Fisher-information conditions. Its channel estimates can be used in
Corollary~\ref{cor:ece-rate-main} when the positive-gap conditions hold.

Prediction information can also be used: divide the prediction range into
pre-specified intervals and jointly estimate their positive-class proportions
and each annotator's channel probabilities, shared across intervals.
The EM algorithm (Joint score-bin EM) and numerical comparisons are in
Appendices~\ref{v03:algorithms} and~\ref{v03:experimental-details}, respectively.

\subsection{Choosing aggregation weights}
\label{sec:annotator-weights}
Channel and class-proportion estimates also determine the aggregation
weights in Section~\ref{v03:estimation}. We design the weights by
minimizing the known-channel residual variance, then substitute the
estimated parameters. For fixed $f$ and nonzero channel gaps, choose
constant $\alpha\in\Delta_K$ to minimize $V_{\mathrm{res}}(\alpha)\coloneqq\Var(\psi_\alpha)$, where
\begin{equation}
 V_{\mathrm{res}}(\alpha)=\Var(Y-Z)+\E[(U_\alpha-Y)^2].
 \label{eq:residual-variance-objective}
\end{equation}
The equality follows from $\E[U_\alpha-Y\mid Y,Z]=0$; only the second term depends on the weights. For each $j\in[K]$, the corrected annotation-noise variance $\bar v_j\coloneqq\E[(U_j-Y)^2]=\E[\Var(U_j\mid Y,Z)]$ satisfies
\[
 \bar v_j=[(1-\mu)a_j(1-a_j)+\mu b_j(1-b_j)]/\delta_j^2.
\]
Thus the channel and class-proportion estimates give a plug-in estimate
of this variance without true labels.

\begin{lemma}[Inverse-variance weights]
\label{lem:iv-main}
Under Assumptions~\ref{v03:marginal} and~\ref{v03:product}, with
$\delta_j\ne0$ for all $j\in[K]$,
$V_{\mathrm{res}}(\alpha)=\Var(Y-Z)+\sum_{j=1}^K\alpha_j^2\bar v_j$.
If every $\bar v_j>0$, the unique minimizer over $\Delta_K$ is
\begin{equation}
 \textstyle\alpha_j^{\IV}\coloneqq
 \bar v_j^{-1}/\bigl(\sum_{k=1}^K\bar v_k^{-1}\bigr),\qquad j\in[K].
 \label{v03:iv}
\end{equation}
\end{lemma}
The proof and zero-variance case are in Appendix~\ref{app:weights}.
For known channels and a fixed evaluation partition, these weights also
minimize the total variance of the unbiased bin-moment estimates
(Proposition~\ref{prop:bin-moment-variance}). In practice, we replace
$a_j,b_j,\mu$ by their estimates and normalize the inverse estimated
variances; Appendix~\ref{app:feasible-weights} specifies this plug-in
rule. \mbox{This optimality result assumes known channels.}

\section{Recalibration}
\label{v03:recalibration}
We fit a measurable post-processing map $g:[0,1]\to[0,1]$ to improve the calibration of $f$.
Let $\calD_\rec\coloneqq\{(Z_i,A_{i,1:K})\}_{i\in[n_\rec]}$ be an i.i.d.
sample from $P_O$, independent of $\calD_\ch$.
We fit $g$ on $\calD_\rec$ using the corrected labels
$\wh U_i=\wh U(A_{i,1:K})$ defined as in
Eq.~\eqref{eq:estimated-corrected-label-main}.

\subsection{Corrected-loss fitting}
We use unbiased loss correction for class-conditional label
noise~\citep{natarajan2013learning}. Logarithmic loss is affine in the
binary label, so $\E[U_\alpha\mid Y,Z]=Y$ recovers the clean-label
expected loss when the channels are known.
We fit the map using corrected labels $\wh U_i$ computed from estimated channels.
\begin{definition}[Corrected logarithmic loss]
\label{def:corrected-loss-main}
Let $\calG$ be a class of measurable maps $[0,1]\to(0,1)$.
For $p\in(0,1)$, $i\in[n_\rec]$, and $g\in\calG$, define the
corrected loss and empirical risk by
\begin{align}
 \wh\ell_i(p)&\coloneqq-\wh U_i\log p-(1-\wh U_i)\log(1-p),\notag\\
 \wh R(g)&\coloneqq n_\rec^{-1}\textstyle\sum_{i=1}^{n_\rec}\wh\ell_i(g(Z_i)).
 \label{v03:loss}
\end{align}
\end{definition}
We minimize this risk over temperature or Platt scaling
maps~\citep{Guo2017,platt1999probabilistic}.
Appendix~\ref{app:parametric-recalibration} defines these families with clipped input
scores and compact parameter sets. These restrictions keep predictions
away from zero and one and prevent an unbounded objective when corrected
labels lie outside $[0,1]$.
\begin{remark}[Risk and calibration guarantees]
\label{rem:risk-calibration}
Minimizing the corrected empirical risk attains clean log risk close to
the best in the chosen family. Theorem~\ref{thm:parametric-recalibration-risk}
bounds the excess risk by sampling, channel-estimation, and optimization
errors. This is a log-risk guarantee, not a guarantee that every fitted
temperature or Platt map reduces $\ECE(g\circ f)$.
\end{remark}

\subsection{Fixed-bin recalibration}
Histogram binning offers a nonparametric alternative \citep{gupta2021distribution}.
The fitting bins need not coincide with the channel-estimation groups or
the evaluation bins.
\begin{definition}[Projected histogram recalibration]
\label{def:histogram-main}
For a positive integer $J$, let $\calH\coloneqq\{H_h\}_{h\in[J]}$ be a
measurable partition of $[0,1]$ chosen independently of $\calD_\rec$.
For $h\in[J]$, set $I_h\coloneqq\{i\in[n_\rec]:Z_i\in H_h\}$ and
$n_h\coloneqq|I_h|$. Write
$\Pi_{[0,1]}(u)\coloneqq\min(1,\max(0,u))$ for projection onto $[0,1]$.
If $n_h>0$, define
\begin{equation}
 \wt m_h\coloneqq n_h^{-1}\textstyle\sum_{i\in I_h}\wh U_i,
 \qquad \wh m_h\coloneqq\Pi_{[0,1]}(\wt m_h).
 \label{v03:hist-est}
\end{equation}
For an empty bin, choose $\wh m_h\in[0,1]$ arbitrarily.
The recalibration map is
$\wh m_\calH(z)\coloneqq\sum_{h=1}^{J}\wh m_h\one(z\in H_h)$.
\end{definition}
The bin mean estimates $\E[Y\mid Z\in H_h]$, and projection makes its output a probability.
\begin{proposition}[Calibration of projected bin averages]
\label{v03:histogram}
Suppose Assumption~\ref{v03:marginal} holds,
$\calD_\rec\indep\calD_\ch$, and
$0\le\wh a_j<\wh b_j\le1$ for $j\in[K]$.
Let $\calH$ and $\wh\alpha\in\Delta_K$ be fixed conditional on
$\calD_\ch$. For every $\eta\in(0,1)$, conditional on $\calD_\ch$,
with probability at least $1-\eta$,
\begin{equation}
 \ECE(\wh m_\calH\circ f)
 \le\dch+(\Bn+1)
 \sqrt{\frac{2(J\log2+\log(4/\eta))}{n_\rec}}.
 \label{v03:hist-bound}
\end{equation}
\end{proposition}
The proof is in Appendix~\ref{v03:recalibration-proofs}.

\section{Experiments}
\label{v03:experiments}
We use three settings: fully synthetic toy data to study ECE estimation
and sample reuse; real tabular data with simulated label noise to test
recalibration; and real images with collected human labels to stress-test
both procedures without guaranteed theoretical assumptions.
Appendix~\ref{v03:experimental-details} gives details and additional results.

\subsection{Toy data}
\label{sec:identification-experiments}
We examine whether noisy labels suffice to estimate the ECE of a fixed
predictor, with either known or estimated noise probabilities.
We generate $X\sim\tfrac12N(-1,1)+\tfrac12N(1,1)$, fix
$f(X)=\sigma(0.3-X)$, and draw $Y\mid Z\sim\Ber(m(Z))$, where
$Z=f(X)$ and $\sigma$ is the logistic function.
Each of three annotators labels every example, independently conditional
on $Y$, with false-positive probabilities $(0.08,0.13,0.18)$ and
false-negative probabilities $(0.16,0.11,0.06)$.
Here $m(z)=P(Y=1\mid Z=z)$ is the true positive-label probability.
Keeping the input distribution, $f$, and noise probabilities unchanged,
we vary only $m$: $m(z)=z$ (Calibrated),
$m(z)=\sigma(0.65\log\tfrac{z}{1-z}-0.3)$ (Shifted), or
$m(z)=\Pi_{[0.02,0.98]}(z+0.12\sin(2\pi z))$ (Nonlinear), where $\Pi$
clips to the indicated interval. The latter two make $f$ miscalibrated;
no recalibration is performed. Figure~\ref{v31:fig-estimation} uses only
Calibrated, whose true ECE is zero; Table~\ref{v31:tab-reuse} uses all three.

\begin{table}[!htb]
\centering\small
\caption{
Toy ECE estimation at $N=8000$: mean $\pm$ SD of absolute error over 200 repetitions ($\times100$).}
\label{v31:tab-reuse}
\setlength{\tabcolsep}{3pt}
\adjustbox{max width=\linewidth}{\begin{tabular}{@{}lrrr@{}}
\toprule
Condition & \shortstack{Sample\\splitting} & \shortstack{Full-sample\\reuse}
& \shortstack{5-fold\\cross-fitting}\\
\midrule
Calibrated & $1.960 \pm 0.480$ & $\boldsymbol{1.343 \pm 0.322}$ & $\boldsymbol{1.343 \pm 0.322}$\\
Shifted & $0.751 \pm 0.605$ & $\boldsymbol{0.448 \pm 0.358}$ & $\boldsymbol{0.448 \pm 0.358}$\\
Nonlinear & $0.618 \pm 0.508$ & $\boldsymbol{0.427 \pm 0.336}$ & $\boldsymbol{0.427 \pm 0.336}$\\
\bottomrule
\end{tabular}}
\end{table}

\newcommand{\TabularRecalibrationRows}{ACSIncome & $6.58 \pm 0.07$ & $12.22 \pm 0.18$ & $3.00 \pm 0.24$ & $12.31 \pm 0.19$ & $2.35 \pm 0.16$ & $9.17 \pm 0.14$ & $\mathbf{0.80 \pm 0.08}$ \\
CreditDefault & $8.43 \pm 0.38$ & $12.50 \pm 0.21$ & $5.54 \pm 0.35$ & $10.09 \pm 0.22$ & $2.37 \pm 0.28$ & $9.06 \pm 0.18$ & $\mathbf{1.51 \pm 0.26}$ \\
HMDA & $6.44 \pm 0.12$ & $15.33 \pm 0.25$ & $4.44 \pm 0.35$ & $14.01 \pm 0.15$ & $1.96 \pm 0.11$ & $11.39 \pm 0.11$ & $\mathbf{0.79 \pm 0.09}$ \\
MEPS & $10.65 \pm 0.85$ & $19.12 \pm 0.66$ & $12.09 \pm 0.92$ & $11.76 \pm 0.25$ & $3.37 \pm 0.43$ & $10.56 \pm 0.33$ & $\mathbf{2.39 \pm 0.39}$ \\
}

\begin{table*}[t]
\centering\small
\caption{
Recalibration on tabular data: dataset-wise Oracle ECE ($\times100$), mean $\pm$ SD over ten seeds.}
\label{v03:tab-tabular}
\setlength{\tabcolsep}{3pt}
\renewcommand{\arraystretch}{1.15}
\adjustbox{max width=\textwidth}{\begin{tabular}{@{}lrrrrrrr@{}}
\toprule
\multirow{2}{*}{Dataset} & \multirow{2}{*}{Base} & \multicolumn{2}{c}{Temperature} & \multicolumn{2}{c}{Platt} & \multicolumn{2}{c}{Fixed-bin}\\
\cmidrule(lr){3-4}\cmidrule(lr){5-6}\cmidrule(lr){7-8}
 & & Naive & Corrected & Naive & Corrected & Naive & Corrected\\
\midrule
\TabularRecalibrationRows
\bottomrule
\end{tabular}}
\end{table*}

\paragraph{ECE estimation.}
Soft-vote ECE and Majority-vote ECE replace the corrected label $\widehat U_i$ in Eq.~\eqref{v03:direct} by $K^{-1}\sum_{j=1}^K A_{ij}$ and $\one\{\sum_{j=1}^K A_{ij}>K/2\}$, respectively, with $K=3$.
\NCECE{} uses true noise probabilities (\emph{given rates}) or estimates obtained from noisy responses using \emph{Joint score-bin EM} (Appendix~\ref{app:em-grouped}); both use equal weights.
ECE uses ten equal-width bins. Clean-label ECE replaces $\widehat U_i$ by $Y_i$ only as a reference.
For $N\in\{500,1000,2000,4000,8000\}$, half the examples estimate noise probabilities; all methods use the remaining half for ECE.
We report mean absolute error (MAE) relative to population ECE over 200 repetitions; error bars show one Monte Carlo standard error.

\begin{figure}[t]
\centering
\includegraphics[width=0.7\linewidth,trim=0bp 7bp 0bp 14bp,clip]{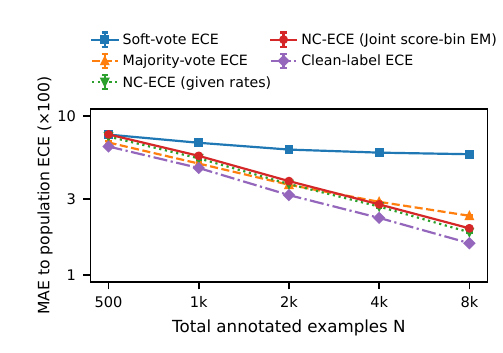}
\caption{Toy ECE-estimation MAE ($\times100$, $m(z)=z$).}
\label{v31:fig-estimation}
\end{figure}

Figure~\ref{v31:fig-estimation} shows decreasing \NCECE{} errors, with
estimated rates performing close to given rates, illustrating
Lemma~\ref{lem:corrected-main} and Theorem~\ref{v03:bound}.
Majority voting is wrong only when at least two of the three annotators are wrong, which is unlikely with low error rates and conditionally independent errors.  These nonzero rates still change
label expectations; correction need not have lower finite-sample MAE.
Appendix 
Figure~\ref{v31:fig-other-functions}
shows the Shifted and
Nonlinear comparisons: voting is close to \NCECE{} in Shifted, whereas
the difference is clearer at larger sample sizes in Nonlinear.

\paragraph{Sample reuse.}
For each of the three conditions, we compare how the same $N$ examples,
each labeled by three annotators, are used:
\emph{sample splitting} (half for noise estimation, half for ECE),
\emph{full-sample reuse} (all examples for both), and
\emph{five-fold cross-fitting} (estimate noise probabilities on four folds
and apply the correction to the remaining fold).
Cross-fitting combines the held-out residuals before computing ECE.
The predictor and bin boundaries remain fixed, and every method uses
Joint score-bin EM and equal weights.

Table~\ref{v31:tab-reuse} shows similar mean errors for reuse and
cross-fitting, both lower than for splitting in all three conditions.
Even when the true ECE is zero, finite-sample estimation
error remains. The calibrated case illustrates the reuse allowed by
Proposition~\ref{cor:reuse-main}, not a theoretical ranking of the methods.

\subsection{Recalibration on tabular data}
We examine whether recalibration improves an existing predictor when
only noisy labels are available for post-hoc fitting. We train decision trees, random forests,
SVMs, naive Bayes, logistic regression, and MLPs with the original labels
of ACSIncome, CreditDefault, HMDA, and MEPS \citep{hansen2024multicalibration}. Each held-out example receives
labels from $K=2,3,4$ simulated annotators, generated by flipping its
original label with annotator- and class-specific probabilities.
Joint score-bin EM estimates the unknown noise probabilities 
on the recalibration sample (Appendix~\ref{v03:realdata}),
and the methods of Section~\ref{v03:recalibration} recalibrate the predictions
using only these noisy labels. Naive uses averaged noisy annotations as targets.
Corrected adjusts each annotation using estimated noise probabilities and combines the corrected labels with inverse-variance weights.
Table~\ref{v03:tab-tabular} checks improvement using original labels of
separate test examples. Corrected Platt and fixed-bin methods reduce mean
ECE relative to Base on all four datasets, whereas corrected temperature scaling worsens
ECE on MEPS. Improvements are not universal.

\subsection{Real label noise: a stress test}
We stress-test ECE evaluation and recalibration on CIFAR-10N, where
class-conditional noise (Assumption~\ref{v03:marginal}) and conditional
independence (Assumption~\ref{v03:product}) are not assured.
Each image has its original CIFAR-10 label and three additional
classifications by distinct people, who can vary across images.
The original multiclass study reports feature-dependent noise
\citep{wei2022learning}. No artificial corruption is introduced.

Both label sets are mapped to animal/vehicle. CLIP ViT-B/32's animal probability is the score \citep{radford2021learning}.
We compare the base predictor, fixed-bin fitting to
mean responses (Naive), and noise-corrected fixed-bin fitting (Corrected).
For recalibration and evaluation, marginal Dawid--Skene EM estimates
noise probabilities from responses alone on separate images.
We reserve original CIFAR-10 labels \citep{krizhevsky2009learning} as benchmark
reference labels for binned \emph{Oracle ECE}; they are not used for fitting or
selection.

\begin{table}[htbp]
\centering\small
\caption{CIFAR-10N fixed-bin recalibration: mean $\pm$ SD over five pilot splits ($\times100$).}
\label{v03:tab-human}
\scriptsize
\setlength{\tabcolsep}{1.8pt}
\begin{tabular}{@{}lrrr@{}}
\toprule
Method & Oracle ECE & Soft-vote ECE & \NCECE{}\\
\midrule
Base & $2.129 \pm 0.000$ & $2.825 \pm 0.000$ & $2.436 \pm 0.000$\\
Naive fixed-bin & $4.351 \pm 0.243$ & $\boldsymbol{0.450 \pm 0.141}$ & $4.695 \pm 0.226$\\
Corrected fixed-bin & $\boldsymbol{0.470 \pm 0.157}$ & $4.291 \pm 0.101$ & $\boldsymbol{0.482 \pm 0.069}$\\
\bottomrule
\end{tabular}
\end{table}

Table~\ref{v03:tab-human} averages five fits sharing held-out samples.
Naive recalibration reduces Soft-vote ECE but worsens reference-label ECE;
correction reverses both changes. In this fixed-bin comparison, \NCECE{}
is close to Oracle ECE values and tracks their direction of
change using noisy responses alone; the noise
assumptions remain unverified.

\section{Conclusion and Limitations}
\label{v03:conclusion}
Corrected annotations enable ECE evaluation and recalibration from noisy responses. Our bounds separate binning, sampling, and
channel-estimation errors, including fixed-bin sample reuse.
Corrected-loss and histogram fitting provide clean-risk and calibration
guarantees, respectively.

\looseness=-1 Our framework assumes score-invariant error rates within each true class.
Identifying unknown channels additionally requires conditional independence
of annotators given the true label and score, and label orientation.
Even under these strong assumptions, small channel gaps can amplify
correction variance, while weak separation between group class proportions
can destabilize two-annotator estimation. 
Likelihood rates need not hold after finitely many EM steps.
Addressing the limitations of these assumptions and extending the framework to multiclass labels remain directions for future work.

\bibliographystyle{plainnat}
\bibliography{main}
\clearpage
\appendix
\allowdisplaybreaks[0]
\numberwithin{equation}{section}
\numberwithin{theorem}{section}
\numberwithin{proposition}{section}
\numberwithin{assumption}{section}
\numberwithin{lemma}{section}
\numberwithin{corollary}{section}
\numberwithin{definition}{section}
\numberwithin{remark}{section}
\numberwithin{example}{section}
\section*{Supplementary Material}
The supplement gives proofs, additional statistical results, algorithms, and
experimental protocols. Appendix~\ref{v03:notation} collects the notation;
Appendix~\ref{v03:settings} discusses the annotation model.
Appendices~\ref{v03:main-proofs}--\ref{v03:recalibration-proofs} establish the
results of Sections~\ref{v03:estimation}--\ref{v03:recalibration}.
Algorithms and experiments are given in Appendices~\ref{v03:algorithms}
and~\ref{v03:experimental-details}, followed by additional related work
and discussion in Appendix~\ref{v03:background}.
\section{Notation}
\label{v03:notation}
The input space is standard Borel. Predictors, post-processing maps, and
bin-membership indicators are measurable, jointly with any fitting data.
A learned predictor is fixed by conditioning on its training data and randomness;
all subsequent samples are independent of that information. We assume the same
$K$ response slots are observed for every example; sparse responses require
additional missingness assumptions and example-dependent aggregation.
Conditional expectations hold almost surely. Pointwise statements about
$m(z)=\E[Y\mid Z=z]$ use a measurable version; regularity assumptions require
a version with the stated property.

\begin{center}
\begin{tabularx}{\textwidth}{@{}p{0.27\textwidth}X@{}}
\toprule
Symbol & Meaning\\
\midrule
$Z=f(X)$, $O=(Z,A_{1:K})$ & Prediction and observed post-hoc data; their laws are $P_Z,P_O$.\\
$a_j,b_j,\delta_j$ & False-positive probability, true-positive probability, and $b_j-a_j$.\\
$\rho_{0,j},\rho_{1,j}$ & Alternative noise coordinates $a_j,1-b_j$.\\
$\mu$, $r(z)$ & $P(Y=1)$ and $m(z)-z$.\\
$\calB,M,\ECE_\calB(f)$ & Evaluation partition, bin count, and population binned error in Eq.~\eqref{v03:binned}.\\
$\calC,L,G,\mu_g$ & Identification partition, group count, group indicator, and $P(Y=1\mid G=g)$. Two-group formulas use $g\in\{0,1\}$.\\
$\calH,J$ & Recalibration partition and its bin count.\\
$\calD_s,n_s$ & Sample and size for $s\in\{\ch,\rec,\val,\ev\}$. Split-sample results use mutually independent samples.\\
$\E_{\calD_s}[h]$ & $n_s^{-1}\sum_{i=1}^{n_s}h(O_i^{(s)})$.\\
$U_j,\psi_j,U_\alpha,\psi_\alpha$ & Known-channel corrected labels and residuals; hats denote fitted corrections.\\
$\Bn,\dch$ & Bounds for a fitted correction and its conditional-mean error.\\
$e_\ch,\varepsilon_{n_\ch}$ & Maximum endpoint error and its deterministic convergence-rate sequence.\\
$V_{\mathrm{res}}(\alpha)$ & Variance of one corrected residual, not of the ECE estimator.\\
\bottomrule
\end{tabularx}
\end{center}

Expectations of fitted functions integrate over a fresh independent example,
with the fit held fixed. Conditional risk and bias statements require finite
integrals for almost every fit, not finite averages over the fitting data.
Positive fitted gaps give a finite $\Bn$ for each fit but need not give
$\E[\Bn]<\infty$. We use the usual $O_p,o_p,\asymp$ notation and natural logarithms.
Norms $\|\cdot\|_2,\|\cdot\|_\infty,\|\cdot\|_\op$ are Euclidean,
maximum-coordinate, and induced operator norms; $DF$, $\nabla$, and $\nabla^2$
denote Jacobian, gradient, and Hessian. Matrix inequalities use the
positive-semidefinite order.

A fixed nonempty informative subset $\calJ\subseteq[K]$ can replace $[K]$ in
correction formulas, with weights in
$\Delta_\calJ=\{\alpha_j\ge0:\sum_{j\in\calJ}\alpha_j=1\}$.
An excluded zero-gap channel is not evaluated as a zero-weighted inverse.
For equal-width partitions, bins are left-closed and right-open, except that
the last bin includes $1$. The same convention applies to $M,L,J$ bins.
Write $\Pi_{[a,b]}(u)=\min(b,\max(a,u))$ and interpret factors with exponent
zero in Bernoulli mass functions as one, including $0^0=1$.

\section{Calibration and the Annotation Model}
\label{v03:settings}
\subsection{Population and empirical binned errors}
For a fixed partition, the clean estimator and its population counterpart are
\begin{align}
 \wh\ECE_\calB(f,S)&=\sum\nolimits_{\ell=1}^M
 \left|\frac1n\sum\nolimits_{i=1}^n(Y_i-Z_i)\one(Z_i\in B_\ell)\right|,
 \label{eq:clean-ece}\\
 \ECE_\calB(f)&=\sum\nolimits_{\ell=1}^M
 \left|\E[(Y-Z)\one(Z\in B_\ell)]\right|.
 \label{v03:binned}
\end{align}
For a positive-mass bin, conditional Jensen gives
$|\E[r(Z)\mid Z\in B_\ell]|\le\E[|r(Z)|\mid Z\in B_\ell]$.
Multiplying by $p_\ell=P(Z\in B_\ell)$ and summing proves
$\ECE_\calB(f)\le\ECE(f)$. Fixed-bin estimation can therefore leave binning
error even as sampling error vanishes.
An independently fitted partition is held fixed in these expectations.
\begin{example}[Cancellation within a bin]
\label{ex:bin-cancellation}
If $Z$ is uniform on $\{1/4,3/4\}$ and $Y\sim\Ber(1/2)$ is independent of $Z$,
then $\ECE(f)=1/4$ but one bin gives $\ECE_\calB(f)=|\E[Y-Z]|=0$.
\end{example}

\subsection{Conditional distributions and identification}
\label{app:assumptions}
Because annotations are binary, Assumption~\ref{v03:marginal} is equivalent to
\begin{equation}\textstyle
 P(A_j=1\mid Y,Z)=a_j+\delta_jY,
 \qquad \delta_j=b_j-a_j.\label{eq:marginal-channel}
\end{equation}
Taking conditional expectations yields
\begin{equation}\textstyle
 q_j(z)=\E[A_j\mid Z=z]=a_j+\delta_jm(z).\label{eq:affine-response}
\end{equation}
Known-channel correction uses this marginal relation and nonzero gaps;
it requires neither positive orientation nor independence among annotators.
Assumption~\ref{v03:product} additionally gives
\begin{equation}\textstyle
 P(A_{1:K}=s\mid Y=y,Z=z)=\prod\nolimits_{j=1}^Kp_j(y)^{s_j}(1-p_j(y))^{1-s_j},
 \qquad p_j(y)=a_j+\delta_jy.\label{eq:product-channel}
\end{equation}
This implies independence of the annotations given $Y$ and
$A_{1:K}\indep Z\mid Y$, producing the likelihood model
\citep{dawid1979maximum}. Dependence can violate this product relation while
preserving the marginal correction identity. If the marginal channels themselves
vary with $Z$ given $Y$, a single shared correction need not remain valid.
Table~\ref{v03:tab-regimes} separates the requirements.
\begin{table}[t]
\centering\small
\caption{Information sufficient for correction or identification.
Unknown-channel estimation also requires label orientation; rate results
impose the regularity conditions in Appendix~\ref{v03:identification-proofs}.}
\label{v03:tab-regimes}
\begin{tabularx}{\linewidth}{@{}p{0.25\linewidth}XX@{}}
\toprule
Task & Annotation model & Additional condition\\
\midrule
Known-channel correction & Assumption~\ref{v03:marginal} & Nonzero gaps for inverted channels\\
Unknown, $K=2$ & Assumptions~\ref{v03:marginal},~\ref{v03:product} & Two informative channels; two positive-mass groups with distinct class proportions strictly between zero and one\\
Unknown, $K\ge3$ & Assumptions~\ref{v03:marginal},~\ref{v03:product} & Three informative annotators; no score separation\\
\bottomrule
\end{tabularx}
\end{table}

\begin{remark}[Label orientation]
\label{rem:orientation-app}
Replacing $\mu$ by $1-\mu$ and exchanging $a_j,b_j$ leaves the marginal response
law unchanged. Identification therefore needs a prescribed orientation,
such as $a_j<b_j$ for a specified informative annotator.
\end{remark}

\subsection{Aggregation before correction}
\label{app:aggregation}
\begin{proposition}[Expectation of an aggregated annotation]
\label{thm:aggregate-affine}
Suppose $P_{A_{1:K}\mid Y,Z}=P_{A_{1:K}\mid Y}$. For a fixed finite-valued
$v:\{0,1\}^K\to\R$, let $c_y=\E[v(A_{1:K})\mid Y=y]$. Then
\begin{equation}\textstyle
 \E[v(A_{1:K})\mid Z]=c_0+(c_1-c_0)m(Z).\label{eq:aggregate-affine}
\end{equation}
The affine expression $c_0+(c_1-c_0)t$ equals $t$ for all $t\in[0,1]$
if and only if $c_0=0,c_1=1$.
\end{proposition}
\begin{proof}
Iterated expectation gives $\E[v(A_{1:K})\mid Z]=\E[c_Y\mid Z]
=(1-m(Z))c_0+m(Z)c_1$. Evaluating the affine expression at $0$ and $1$
proves the last assertion.
\end{proof}
The proposition covers majority vote and fixed weighted averages.
For fixed $K$, more evaluation examples do not remove a conditional-mean
mismatch, although particular distributions may have equal scalar ECEs.
\begin{remark}[Prediction-dependent aggregation]
\label{rem:score-posterior}
For the correctly specified posterior $\tau(A_{1:K},Z)=\E[Y\mid A_{1:K},Z]$,
\begin{equation}\textstyle
 \E[\tau(A_{1:K},Z)\mid Z]=m(Z).\label{eq:posterior-mean}
\end{equation}
This follows by iterated expectation. Such prediction-dependent rules are not
excluded by Proposition~\ref{thm:aggregate-affine}; fitted posteriors require
their own estimation-error analysis.
\end{remark}

\section{Calibration Estimation: Statements and Proofs}
\label{v03:main-proofs}
\subsection{Weighted representations and corrected moments}
\label{app:weighted-derivations}
A bounded function $w$ selects or weights predictions. Bin indicators
recover the usual ECE moments. Allowing one sign in each bin expresses
the sum of absolute moments as a maximum, which is useful for concentration
bounds; it does not change how the estimator is computed.
Let $\calW_\infty$ be the class of measurable maps $[0,1]\to[-1,1]$ and
define the calibration moment
\begin{equation}\textstyle
 M_f(w)\coloneqq\E[(Y-Z)w(Z)],\qquad w\in\calW_\infty.
 \label{eq:weighted-moment}
\end{equation}
For $\calB\coloneqq\{B_\ell\}_{\ell\in[M]}$, let
$r_\ell\coloneqq\E[(Y-Z)\one(Z\in B_\ell)]$,
$w_s(z)\coloneqq\sum_{\ell=1}^{M}s_\ell\one(z\in B_\ell)$, and
$\calW_\calB\coloneqq\{w_s:s\in[-1,1]^M\}$.
\begin{lemma}[Weighted calibration representations]
\label{lem:weighted-representation}
For the fixed predictor $f$,
\begin{align}
 \ECE(f)&=\sup\nolimits_{w\in\calW_\infty}M_f(w)
         =\sup\nolimits_{w\in\calW_\infty}|M_f(w)|,\label{eq:weighted-ece}\\
 \ECE_\calB(f)&=\sup\nolimits_{w\in\calW_\calB}M_f(w)
         =\max\nolimits_{s\in\{-1,1\}^M}\sum\nolimits_{\ell=1}^{M}s_\ell r_\ell.
 \label{eq:binned-signs}
\end{align}
\end{lemma}
\begin{proof}
Since $w(Z)$ is measurable with respect to $Z$,
$M_f(w)=\E[(m(Z)-Z)w(Z)]$. Its absolute value is at most
$\E[|m(Z)-Z|]$. Equality is attained by the measurable sign of $m(z)-z$,
with any value in $[-1,1]$ at zero. The class is symmetric under sign
change, which gives both suprema. For a binwise-constant $w_s$,
$M_f(w_s)=\sum_{\ell=1}^{M}s_\ell r_\ell$.
Choosing each sign to agree with $r_\ell$ gives
$\sum_{\ell=1}^{M}|r_\ell|$, including zero-moment bins.
\end{proof}

\begin{lemma}[Corrected moments and unbiased empirical moments]
\label{thm:corrected-moment}
Suppose Assumption~\ref{v03:marginal} holds and $\delta_j\ne0$ for
$j\in[K]$. For $U_j\coloneqq(A_j-a_j)/\delta_j$ and
$\psi_j\coloneqq U_j-Z$, and every fixed $\alpha\in\Delta_K$,
\begin{align}
 \E[U_j\mid Y,Z]&=Y,\qquad \E[U_\alpha\mid Y,Z]=Y,\label{eq:correction-conditional}\\
 \E[\psi_j\mid Z]&=\E[\psi_\alpha\mid Z]=m(Z)-Z,\label{eq:residual-conditional}\\
 \E[\psi_jw(Z)]&=\E[\psi_\alpha w(Z)]=M_f(w),
 \quad w\in\calW_\infty.\label{eq:corrected-moment}
\end{align}
For an independent evaluation sample, the empirical corrected moment is
unbiased, conditional on any independent data determining $\alpha,w$.
\end{lemma}
\begin{proof}[Proof of Lemmas~\ref{lem:corrected-main} and~\ref{thm:corrected-moment}]
Eq.~\eqref{eq:marginal-channel} gives
$\E[(A_j-a_j)/\delta_j\mid Y,Z]=Y$.
Linearity and $\sum_{j=1}^{K}\alpha_j=1$ give the aggregate identity.
Conditioning again on $Z$ gives
$\E[U_j\mid Z]=\E[U_\alpha\mid Z]=m(Z)$. Subtracting $Z$ proves
Eq.~\eqref{eq:residual-conditional}. Substitution into the ECE definition
in Eq.~\eqref{v03:population} proves Eq.~\eqref{v03:identity}.
Multiplication by bounded $w(Z)$ and expectation yield
Eq.~\eqref{eq:corrected-moment}. All variables are integrable because the
finite set of inverted gaps is nonzero. Finally, expectation of an
empirical average equals its population expectation when the function is
fixed independently of the sample. No joint annotator independence is used.
\end{proof}

\subsection{Independent-sample ECE bound}
\label{v03:full-ece}
Write $\wh\psi(O)=\wh U(A_{1:K})-Z$ and define
\begin{equation}\textstyle
 \wh r_\ell\coloneqq\E_{\calD_\ev}[\wh\psi\one(Z\in B_\ell)],
 \qquad \wh\ECE^{\nc}_\calB(f,\calD_\ev)=\sum\nolimits_{\ell=1}^{M}|\wh r_\ell|.
 \label{eq:direct-ece-estimator}
\end{equation}
The estimates, weights, and bins are measurable functions of
$\calD_\ch$, which may include independent pilot information.
They are independent of $\calD_\ev$.

\begin{lemma}[Conditional-expectation error and bounds]
\label{lem:mean-error}
Under Assumption~\ref{v03:marginal}, assume
$0\le\wh a_j<\wh b_j\le1$ for $j\in[K]$ and
$\wh\alpha\in\Delta_K$.
For an independent example $O\sim P_O$, with $\calD_\ch$ held fixed
as in Appendix~\ref{v03:notation},
\begin{align}
 b_\ch(z)&\coloneqq \E[\wh U\mid Z=z,\calD_\ch]-m(z)\notag\\
 &=\sum\nolimits_{j=1}^{K}\wh\alpha_j
 \frac{(1-m(z))(a_j-\wh a_j)+m(z)(b_j-\wh b_j)}{\wh\delta_j},
 \label{eq:conditional-channel-bias}\\
 |b_\ch(Z)|&\le\dch,\qquad |\wh U|\le\Bn,
 \qquad |\wh\psi|\le\Bn.
 \label{eq:estimated-envelope}
\end{align}
These bounds also hold for zero true gaps. In that case they need not be
small, and they do not imply that the channel can be identified.
\end{lemma}
\begin{proof}
Independence of the example and $\calD_\ch$
allows the estimates to be held fixed. By
Eq.~\eqref{eq:affine-response}, the conditional expectation of the $j$th estimated
corrected label is $(a_j+\delta_j m(z)-\wh a_j)/\wh\delta_j$.
Subtracting $m(z)$ and using
$\delta_j-\wh\delta_j=(b_j-\wh b_j)-(a_j-\wh a_j)$ proves
Eq.~\eqref{eq:conditional-channel-bias}.
Its numerator is a convex combination of two endpoint errors, so its
absolute value is bounded by their maximum. Summing gives $\dch$.
The label bound follows from $|A_j-\wh a_j|\le1$.
For the residual, note that
$\wh a_j+\wh\delta_jZ=(1-Z)\wh a_j+Z\wh b_j\in[0,1]$.
Thus
$|A_j-\wh a_j-\wh\delta_jZ|\le1$; division and convex aggregation
prove the last bound.
\end{proof}

\begin{theorem}[Finite-sample error and conditional bias]
\label{thm:arbitrary-rho-bound}
Suppose Assumption~\ref{v03:marginal} holds. Let $\calD_\ch$ and
$\calD_\ev$ be independent i.i.d. samples from $P_O$, with the predictor
fixed independently of both. Let the estimates, weights and partition
satisfy Definition~\ref{def:corrected-binning-ece}, and set
$p_\ell\coloneqq P(Z\in B_\ell\mid\calD_\ch)$.
No smoothness is needed for the first two assertions.
Conditional on $\calD_\ch$,
\begin{equation}
 -\dch\le
 \E[\wh\ECE^{\nc}_\calB(f,\calD_\ev)\mid\calD_\ch]-\ECE_\calB(f)
 \le \Bn\sqrt{M/n_\ev}+\dch.
 \label{eq:bias-general}
\end{equation}
For every $\eta\in(0,1)$, with conditional probability at least $1-\eta$,
\begin{equation}
 \abs{\wh\ECE^{\nc}_\calB(f,\calD_\ev)-\ECE_\calB(f)}
 \le\Bn\sqrt{\frac{2(M\log2+\log(2/\eta))}{n_\ev}}+\dch.
 \label{eq:full-ece-bound}
\end{equation}
If $r(z)\coloneqq m(z)-z$ has an $L_r$-Lipschitz version on the support of
$P_Z$, and $h_\calB\coloneqq\max_{\ell\in[M]}\operatorname{diam}(B_\ell)$, then
\begin{equation}
 0\le\ECE(f)-\ECE_\calB(f)\le L_rh_\calB.
 \label{eq:bin-approximation}
\end{equation}
Combining Eq.~\eqref{eq:full-ece-bound} and Eq.~\eqref{eq:bin-approximation}
proves Theorem~\ref{v03:bound} for the population target.
\end{theorem}
\begin{proof}
Condition throughout on $\calD_\ch$. Set
$q_\ell=\E[\wh\psi\one(Z\in B_\ell)\mid\calD_\ch]$ and
$Q_\calB=\sum_\ell|q_\ell|$. Lemma~\ref{lem:mean-error} gives
\begin{equation}\textstyle
 \sum\nolimits_\ell|q_\ell-r_\ell|
 \le\sum\nolimits_\ell\E[|b_\ch(Z)|\one(Z\in B_\ell)\mid\calD_\ch]\le\dch.
 \label{eq:population-bin-perturbation}
\end{equation}
Thus $|Q_\calB-\ECE_\calB(f)|\le\dch$.
Since $\E[\wh r_\ell]=q_\ell$ and
$\Var(\wh r_\ell)\le\Bn^2p_\ell/n_\ev$, Jensen and Cauchy--Schwarz yield
$0\le\E[\wh\ECE^{\nc}_\calB]-Q_\calB \le\sum_\ell\E|\wh r_\ell-q_\ell| \le\Bn n_\ev^{-1/2}\sum_\ell\sqrt{p_\ell} \le\Bn\sqrt{M/n_\ev},$
proving Eq.~\eqref{eq:bias-general}. For the deviation bound,
\begin{equation}\textstyle
 \sum\nolimits_\ell|\wh r_\ell-q_\ell|
 =\max\nolimits_{s\in\{-1,1\}^M}|(\E_{\calD_\ev}-\E)[\wh\psi w_s(Z)]|.
 \label{eq:empirical-signs}
\end{equation}
Each of the $2^M$ functions is bounded by $\Bn$ in absolute value.
Two-sided Hoeffding and a union bound give tail probability at most
$2^{M+1}\exp(-n_\ev t^2/(2\Bn^2))$.
Setting this equal to $\eta$ and using Eq.~\eqref{eq:population-bin-perturbation}
proves Eq.~\eqref{eq:full-ece-bound}.
Finally, in each positive-mass bin let $\bar r_\ell=\E[r(Z)\mid Z\in B_\ell]$.
An independent conditional copy $Z'$ gives
$|r(z)-\bar r_\ell|\le\E[|r(z)-r(Z')|\mid Z'\in B_\ell]
\le L_r\operatorname{diam}(B_\ell)$ almost surely.
Conditional Jensen and summation with bin masses prove
$0\le\ECE(f)-\ECE_\calB(f)\le L_rh_\calB$ and Theorem~\ref{v03:bound}.
\end{proof}

\begin{lemma}[Known-channel bias]
\label{lem:known-bias}
Suppose Assumption~\ref{v03:marginal} holds and $\delta_j\ne0$ for
$j\in[K]$. Set
\begin{equation}\textstyle
 B_\rho\coloneqq\max\nolimits_{j\in[K]}|\delta_j|^{-1}.
 \label{eq:B-rho}
\end{equation}
Let $\alpha\in\Delta_K$ and $\calB$ be fixed. On an i.i.d. sample
$\calD=\{O_i\}_{i\in[n]}\sim P_O^n$, define the known-channel estimator
$\wh\ECE^{\nc}_{\rho,\calB}(f,\calD) \coloneqq \sum_{\ell=1}^{M}\left|\E_\calD[ (U_\alpha-Z)\one(Z\in B_\ell)]\right|.$
Each corrected bin moment is unbiased, and
\begin{equation}
 0\le\E[\wh\ECE^{\nc}_{\rho,\calB}(f,\calD)]-\ECE_\calB(f)
 \le B_\rho\sqrt{M/n}.
 \label{eq:known-bias}
\end{equation}
If the weights and partition are chosen using independent data, the same
statement holds conditional on those data.
\end{lemma}
\begin{proof}
Lemma~\ref{thm:corrected-moment} gives unbiased bin moments.
Since $a_j+\delta_jZ\in[0,1]$,
$|\psi_j|\le|\delta_j|^{-1}$, including negative nonzero gaps.
Thus $|\psi_\alpha|\le B_\rho$.
The Jensen and variance argument in the preceding proof, now with
$q_\ell=r_\ell$, yields the result.
\end{proof}

\subsubsection{Proof of Corollary~\ref{cor:ece-rate-main}}
\label{app:ece-rates}
\begin{proof}
Use $e_\ch$ from Eq.~\eqref{eq:channel-estimation-error}.
The assumption $e_\ch=O_p(\varepsilon_{n_\ch})$ with
$\varepsilon_{n_\ch}\to0$ implies $e_\ch=o_p(1)$.
For every $j\in[K]$,
$|\wh\delta_j-\delta_j| \le |\wh a_j-a_j|+|\wh b_j-b_j|\le e_\ch.$
Thus, on the event $e_\ch\le\delta_{\min}/2$, whose probability tends to
one, $\wh\delta_j\ge\delta_{\min}/2$ for all $j$, and
$\Bn\le2/\delta_{\min},\quad \dch\le2e_\ch/\delta_{\min}.$
The conditional probability bound in Theorem~\ref{v03:bound}, integrated
over $\calD_\ch$, therefore gives a sampling term of order
$O_p(\sqrt{M/n_\ev})$ and channel error $O_p(\varepsilon_{n_\ch})$.
For equal-width bins, $h_\calB=M^{-1}$, proving Eq.~\eqref{v03:rate}.
The choice $M\asymp n_\ev^{1/3}$ balances binning and sampling.
For the reused sample, the same event and deterministic inequalities
apply with $e_\ch=O_p(\varepsilon_n)$.
Apply Proposition~\ref{cor:reuse-main}; fixed $K$ changes only a constant
inside the logarithm. This gives
$O_p(M^{-1}+\sqrt{M/n}+\varepsilon_n)$ and the stated choice of $M$.
Neither argument treats fitted coefficients as independent of the sample
that fitted them.
\end{proof}
The smoothness-free bound in Theorem~\ref{thm:arbitrary-rho-bound} also
gives $O_p(n_\ev^{-1/2}+\varepsilon_{n_\ch})$ error relative to
$\ECE_\calB(f)$ for a fixed partition. That statement concerns the
auxiliary binned quantity, not $\ECE(f)$.

\begin{remark}[Evaluation bins and channel groups]
\label{rem:bin-growth}
A growing evaluation bin count $M$ need not change the dimension of the
channel estimator: its group count $L$ can remain fixed. The rates above
require the stated channel-error rate; fixed-dimensional channel theory
does not supply that rate when its group count also grows.
\end{remark}

\subsection{Same-sample estimation with fixed bins}
\label{v03:full-same-sample}
A uniform finite-class argument permits the same observations to fit the
channels, choose the weights, and evaluate ECE. It controls the empirical
deviations simultaneously over all admissible fitted corrections.
Population expectations of a fitted function always refer to an independent
example, with its realized coefficients held fixed.

\begin{corollary}[Same-sample ECE estimation]
\label{cor:same-sample-ece}
Let $\calD=\{O_i\}_{i\in[n]}$ be i.i.d. from $P_O$, and let the
predictor and partition $\calB$ be fixed independently of this sample.
Suppose Assumption~\ref{v03:marginal} holds. Let
$0\le\wh a_j<\wh b_j\le1$ for $j\in[K]$ and
$\wh\alpha\in\Delta_K$ be arbitrary measurable functions of $\calD$.
Define $\Bn,\dch$ as in Theorem~\ref{v03:bound} for these estimates.
For every $\eta\in(0,1)$, with probability at least $1-\eta$,
\begin{equation}
 \abs{\wh\ECE^{\nc}_\calB(f,\calD)-\ECE_\calB(f)}
 \le\Bn\sqrt{\frac{2(M\log2+\log(6K/\eta))}{n}}+\dch.
 \label{eq:same-sample-bound}
\end{equation}
No nonzero-true-gap assumption is needed for this inequality.
If $r(z)\coloneqq m(z)-z$ has an $L_r$-Lipschitz version, add $L_rh_\calB$
to obtain a bound relative to $\ECE(f)$.
\end{corollary}
\begin{proof}
For $j\in[K]$, define three fixed functions
$h_{j,0}(O)\coloneqq A_j$, $h_{j,1}(O)\coloneqq A_j-Z$, and
$h_{j,2}(O)\coloneqq A_j-1$.
For every $s\in\{-1,1\}^M$ and $v\in\{0,1,2\}$,
$h_{j,v}w_s$ takes values in $[-1,1]$.
Hoeffding's inequality and a union bound over the $3K2^M$ functions give
an event of probability at least $1-\eta$ on which
\begin{equation}\textstyle
 \max\nolimits_{j,v,s}|(\E_\calD-\E)[h_{j,v}w_s]|
 \le t_n,\qquad
 t_n\coloneqq\sqrt{2(M\log2+\log(6K/\eta))/n}.
 \label{eq:reuse-finite-class}
\end{equation}
For any $0\le a'<b'\le1$, the exact decomposition
\begin{equation}\textstyle
 A_j-a'-(b'-a')Z
 =(1-b')h_{j,0}+(b'-a')h_{j,1}+a'h_{j,2}
 \label{eq:residual-convex-hull}
\end{equation}
has nonnegative coefficients summing to one.
Consequently, on the same event, simultaneously for all admissible
channel estimates and convex weights,
$\max_s|(\E_\calD-\E)[\wh\psi w_s]| \le\sum_{j=1}^{K}\frac{\wh\alpha_j}{\wh\delta_j}t_n =\Bn t_n.$
The inequality holds for every realized sample on the uniform event, even
when the coefficients are fitted on $\calD$.
For each realized correction, the algebra of
Eq.~\eqref{eq:conditional-channel-bias}, applied to an independent copy,
also gives
$|\E[\wh\psi w_s]-M_f(w_s)|\le\dch$.
Thus $\max_s|\E_\calD[\wh\psi w_s]-M_f(w_s)|\le\Bn t_n+\dch$.
The difference between the two maxima defining the empirical and
population binned errors is bounded by this maximum deviation.
Eq.~\eqref{eq:bin-approximation} gives the population extension and proves
Proposition~\ref{cor:reuse-main}.
\end{proof}
\begin{remark}[Failed fits and data reuse]
\label{rem:failed-fits}
The finite-class event is uniform over all admissible estimates.
If a procedure can fail to return admissible gaps, the probability that it
both succeeds and violates the bound is at most $\eta$.
This is not a conditional coverage claim restricted to successful fits.
Alternatively, a fixed admissible fallback defines the estimator on all
samples, with fitting failures reported separately.
The corollary does not cover a partition or base predictor fitted on
$\calD$.

\end{remark}

\begin{corollary}[Explicit channel-error radius]
\label{cor:explicit-channel-radius}
Under the conditions of Corollary~\ref{cor:same-sample-ece},
suppose $\delta_j\ge\delta_{\min}>0$ for $j\in[K]$ and
\begin{equation}\textstyle
 P\left(\max\nolimits_{j\in[K]}(|\wh a_j-a_j|+|\wh b_j-b_j|)
       \le r_n\right)\ge1-\eta_\ch,
 \qquad r_n\le\delta_{\min\nolimits}/2.
 \label{eq:same-channel-event}
\end{equation}
Then, for $\eta_\ev\in(0,1)$, with probability at least
$1-\eta_\ch-\eta_\ev$,
\begin{equation}
 \abs{\wh\ECE^{\nc}_\calB(f,\calD)-\ECE_\calB(f)}
 \le\frac{2}{\delta_{\min\nolimits}}
 \left(\sqrt{\frac{2(M\log2+\log(6K/\eta_\ev))}{n}}+r_n\right).
 \label{eq:same-sample-radius-bound}
\end{equation}
\end{corollary}
\begin{proof}
On the channel event, $\wh\delta_j\ge\delta_{\min}/2$,
$\Bn\le2/\delta_{\min}$, and $\dch\le2r_n/\delta_{\min}$.
Substitute these bounds into Corollary~\ref{cor:same-sample-ece}
and take a union bound with the channel event.
\end{proof}
The bound is linear in the inverse gap multiplying the channel radius.
For fixed $K$, consistent root-$n$ channel estimates, a Lipschitz residual,
and equal-width bins with $M\asymp n^{1/3}$, it yields
\begin{equation}
 \abs{\wh\ECE^{\nc}_\calB(f,\calD)-\ECE(f)}=O_p(n^{-1/3}).
 \label{eq:same-sample-rate}
\end{equation}
This rate requires smoothness and does not establish distribution-free or minimax optimality.

\subsection{Cross-fitting with one common evaluation partition}
\label{app:cross-fitting}
Let $\calD=\{O_i\}_{i\in[n]}$ consist of i.i.d. examples from $P_O$.
Let the index sets $I_1,\ldots,I_F$ partition $[n]$ independently of the
data, with $F\ge2$ and $n_k\coloneqq|I_k|>0$ for $k\in[F]$.
Fix a common partition $\calB$ independently of the full sample.
For each fold $k$, fit channels and weights on
$\calD_{-k}\coloneqq\{O_i:i\in[n]\setminus I_k\}$, obtaining
$0\le\wh a_j^{(-k)}<\wh b_j^{(-k)}\le1$ and
$\wh\alpha^{(-k)}\in\Delta_K$.
For each $k\in[F]$, define
$\wh\psi^{(-k)}(O) \coloneqq\sum_{j=1}^{K}\wh\alpha_j^{(-k)} \frac{A_j-\wh a_j^{(-k)}}{\wh b_j^{(-k)}-\wh a_j^{(-k)}}-Z.$
The held-out bin moments and their aggregate are
\begin{equation}\textstyle
 \wh r_{k,\ell}\coloneqq\frac1{n_k}\sum\nolimits_{i\in I_k}
       \wh\psi^{(-k)}(O_i)\one(Z_i\in B_\ell),\quad
 \wh r_\ell^{\mathrm{cf}}\coloneqq\sum\nolimits_{k=1}^{F}\frac{n_k}{n}\wh r_{k,\ell},
 \quad \wh\ECE_\calB^{\mathrm{cf}}(f,\calD)
 \coloneqq\sum\nolimits_{\ell=1}^{M}|\wh r_\ell^{\mathrm{cf}}|.
 \label{eq:crossfit-estimator}
\end{equation}
The absolute values are taken after aggregating the moments, not before.
For each $k\in[F]$, let $\Bn^{(-k)}$ and $d_\ch^{(-k)}$ be the quantities
$\Bn$ and $\dch$ in Eq.~\eqref{v03:channel-error}, computed using the
channels and weights fitted without fold $k$.
\begin{corollary}[Cross-fitted calibration error]
\label{cor:crossfit}
Under Assumption~\ref{v03:marginal} and the preceding conditions,
for $\eta\in(0,1)$, with probability at least $1-\eta$,
\begin{equation}
 \abs{\wh\ECE_\calB^{\mathrm{cf}}(f,\calD)-\ECE_\calB(f)}
 \le\sum\nolimits_{k=1}^{F}\frac{n_k}{n}
 \left(\Bn^{(-k)}\sqrt{\frac{2(M\log2+\log(2F/\eta))}{n_k}}+d_\ch^{(-k)}\right).
 \label{eq:crossfit-bound}
\end{equation}
An $L_rh_\calB$ term gives the population extension under the same
smoothness condition as Theorem~\ref{v03:bound}.
\end{corollary}
\begin{proof}
Hold the independently chosen fold assignment and common partition fixed.
For a fixed $k$, condition only on $\calD_{-k}$.
The held-out fold is an independent i.i.d. sample.
The proof of Theorem~\ref{thm:arbitrary-rho-bound}, before taking the
difference of scalar errors, yields
$\sum_{\ell=1}^{M}|\wh r_{k,\ell}-r_\ell| \le \Bn^{(-k)}\sqrt{2(M\log2+\log(2F/\eta))/n_k}+d_\ch^{(-k)}$
with conditional probability at least $1-\eta/F$.
Integrating this conditional probability gives the same unconditional
failure bound. A union bound over the $F$ folds is valid without
independence of their fitted estimators. On the simultaneous event,
$\sum_{\ell=1}^{M}|\wh r_\ell^{\mathrm{cf}}-r_\ell| \le\sum_{k=1}^{F}\frac{n_k}{n} \sum_{\ell=1}^{M}|\wh r_{k,\ell}-r_\ell|,$
which proves the claim by a reverse triangle inequality.
\end{proof}
\begin{remark}[Common bins and the scope of cross-fitting]
\label{rem:crossfit-bins}
The common partition ensures that every fold estimates the same $r_\ell$.
Arbitrary fold-specific partitions do not satisfy that requirement.
Nor does this result establish calibration of a single map formed by
combining independently fitted fold-specific recalibration maps.
The fixed-bin recalibration theorem retains its stated sample splitting.

\end{remark}

\subsection{\NCECE{} and histogram plug-in estimators}
\label{app:direct-plugin}
Let $\calD=\{O_i\}_{i\in[n]}$ be the evaluation sample, and set
$n_\ell\coloneqq\sum_{i=1}^{n}\one(Z_i\in B_\ell)$ for $\ell\in[M]$.
For a nonempty bin, let
$\bar U_\ell\coloneqq n_\ell^{-1}\sum_{i:Z_i\in B_\ell}\wh U_i$ and
$\bar Z_\ell\coloneqq n_\ell^{-1}\sum_{i:Z_i\in B_\ell}Z_i$.
\begin{lemma}[Equivalent \NCECE{} formulas]
\label{lem:direct-plugin}
The following equality holds:
\begin{equation}\textstyle
 \sum\nolimits_{\ell:n_\ell>0}\frac{n_\ell}{n}
       |\bar U_\ell-\bar Z_\ell|
 =\sum\nolimits_{\ell=1}^{M}\left|\frac1n\sum\nolimits_{i=1}^{n}
       (\wh U_i-Z_i)\one(Z_i\in B_\ell)\right|
 \label{eq:direct-plugin-equality}
\end{equation}

Empty bins contribute zero to both sides.
\end{lemma}
\begin{proof}
For each nonempty bin, distribute $n_\ell/n$ inside the absolute value
and substitute the definitions of its two means. Empty bins have zero
sample mass, so both expressions contribute zero there.
\end{proof}
For constant weights and channels, averaging annotations within a bin and
then applying their affine corrections gives exactly $\bar U_\ell$.
Thus a histogram plug-in computation is not intrinsically a different
estimator. Projection, a different fitting sample, different bins, or a
different estimator of the conditional expectation can change the result,
but those changes must be
specified in a comparison. No uniform ranking of \NCECE{} and plug-in
estimators follows from the moment analysis.

\subsection{Aggregation weights and residual variance}
\label{app:weights}
With the predictor, known channels, and constant weights fixed,
$V_{\mathrm{res}}(\alpha)=\Var(\psi_\alpha)$ is the variance over a fresh example.
Let Assumptions~\ref{v03:marginal} and~\ref{v03:product} hold and suppose the channels used in
correction have nonzero gaps. Define $U_j\coloneqq(A_j-a_j)/\delta_j$,
$\mu\coloneqq P(Y=1)$, and
\begin{align}
 v_j(z)&\coloneqq\frac{(1-m(z))a_j(1-a_j)+m(z)b_j(1-b_j)}{\delta_j^2},
 \label{eq:local-variance}\\
 \bar v_j&\coloneqq\E[v_j(Z)]=\E[(U_j-Y)^2]
 =\frac{(1-\mu)a_j(1-a_j)+\mu b_j(1-b_j)}{\delta_j^2}.
 \label{eq:average-variance}
\end{align}
\begin{lemma}[Optimal constant weights for known channels]
\label{lem:variance-optimal-alpha}
For $\alpha\in\Delta_K$,
\begin{equation}\textstyle
 \Var(U_\alpha-Z)=\Var(Y-Z)+\sum\nolimits_{j=1}^{K}\alpha_j^2\bar v_j.
 \label{eq:variance-decomp}
\end{equation}
If every $\bar v_j>0$, the unique minimizing constant vector is
\begin{equation}\textstyle
 \alpha_j^{\IV}\coloneqq\frac{\bar v_j^{-1}}{\sum\nolimits_{k=1}^{K}\bar v_k^{-1}},
 \qquad j\in[K].
 \label{eq:iv-weights}
\end{equation}
If some $\bar v_j=0$, the minimizers are precisely the convex vectors
supported on zero-variance annotators.
\end{lemma}
\begin{proof}[Proof of Lemmas~\ref{lem:iv-main} and~\ref{lem:variance-optimal-alpha}]
The decomposition $U_\alpha-Z=(Y-Z)+(U_\alpha-Y)$ has uncorrelated
terms because $\E[U_\alpha-Y\mid Y,Z]=0$.
Conditional independence gives
$\Var(U_\alpha\mid Y,Z)=\sum_{j=1}^{K}\alpha_j^2\Var(U_j\mid Y,Z)$.
Averaging the two Bernoulli conditional variances yields
Eq.~\eqref{eq:variance-decomp}.
For positive $\bar v_j$, Cauchy--Schwarz gives
$1=\left(\sum_{j=1}^{K}\alpha_j\right)^2 \le\left(\sum_{j=1}^{K}\alpha_j^2\bar v_j\right) \left(\sum_{j=1}^{K}\bar v_j^{-1}\right),$
with equality exactly when $\alpha_j\bar v_j$ is constant.
Normalization gives Eq.~\eqref{eq:iv-weights}, and strict convexity proves uniqueness. 
If some $\bar v_j=0$, the noise term is minimized exactly by assigning
zero weight to annotators with positive variance.
\end{proof}
\paragraph{Alternative constant rules.}
For positive gaps, Equal, Gap, and Squared-gap weights are proportional to
$1,\delta_j,\delta_j^2$. Writing
$\zeta_j=(1-\mu)a_j(1-a_j)+\mu b_j(1-b_j)$ gives
$\alpha_j^{\IV}\propto\delta_j^2/\zeta_j$, so Squared-gap and IV agree when
all positive $\zeta_j$ are equal. These rules compare residual variance,
not finite-sample ECE mean squared error. Plug-in IV and shrinkage are specified
in Appendix~\ref{app:feasible-weights}.

\subsubsection{Connection to bin-moment and \NCECE{} variance}
\label{app:bin-moment-variance}
Fix the predictor, the true nonzero-gap channels, a partition
$\calB\coloneqq\{B_1,\ldots,B_M\}$, and constant weights
$\alpha\in\Delta_K$ independently of an i.i.d. evaluation sample
$\calD\coloneqq\{O_i\}_{i=1}^n$. For $\ell\in[M]$, define
\begin{align}
 R_{\ell,\alpha}(O)&\coloneqq\psi_\alpha(O)\one(Z\in B_\ell),\qquad
 r_\ell\coloneqq\E[(Y-Z)\one(Z\in B_\ell)],\notag\\
 \wh r_{\ell,\alpha}&\coloneqq\frac1n\sum\nolimits_{i=1}^nR_{\ell,\alpha}(O_i),
 \qquad \wh E_\alpha\coloneqq\sum\nolimits_{\ell=1}^M|\wh r_{\ell,\alpha}|.
 \label{eq:variance-bin-estimates}
\end{align}
Here $\wh E_\alpha$ is \NCECE{} with known channels.
Unbiasedness applies to each $\wh r_{\ell,\alpha}$, not generally to
$\wh E_\alpha$; the absolute values give
$\E[\wh E_\alpha]\ge\sum_{\ell=1}^M|r_\ell|$.

\begin{proposition}[Variance of bin moments and \NCECE{}]
\label{prop:bin-moment-variance}
Under Assumption~\ref{v03:marginal} and the preceding conditions,
$\E[\wh r_{\ell,\alpha}]=r_\ell,$
\begin{align}
 \sum\nolimits_{\ell=1}^M\Var(\wh r_{\ell,\alpha})
 &=\frac1n\left(\E[\psi_\alpha^2]-\sum\nolimits_{\ell=1}^M r_\ell^2\right)\notag\\
 &=\frac1n\left(V_{\mathrm{res}}(\alpha)
       +(\E[Y-Z])^2-\sum\nolimits_{\ell=1}^M r_\ell^2\right).
 \label{eq:bin-variance-sum}
\end{align}
Thus minimizing $V_{\mathrm{res}}$ is equivalent to minimizing the sum
of variances, or the sum of mean squared errors, of these bin-moment estimates.

If $M$ and $\calB$ remain fixed as $n\to\infty$, and $r_\ell\ne0$
for every bin with $P(Z\in B_\ell)>0$, then
\begin{equation}\textstyle
 \Var(\wh E_\alpha)=\frac{\sigma_\alpha^2}{n}+o(n^{-1}),\qquad
 \sigma_\alpha^2\coloneqq V_{\mathrm{res}}(\alpha)
       +(\E[Y-Z])^2-\ECE_\calB(f)^2.
 \label{eq:bin-ece-asymptotic-variance}
\end{equation}
Consequently, under Assumption~\ref{v03:product} as well, the IV weights
of Lemma~\ref{lem:variance-optimal-alpha} minimize both the exact total
bin-moment variance and, under the nonzero-moment condition, the leading
asymptotic variance of known-channel \NCECE{} among constant convex weights.
\end{proposition}
\begin{proof}
The corrected-moment identity gives $\E[R_{\ell,\alpha}]=r_\ell$.
Independent evaluation examples imply
$\Var(\wh r_{\ell,\alpha})=\Var(R_{\ell,\alpha})/n$.
Since the bins partition the score space,
$\sum_{\ell=1}^M R_{\ell,\alpha}^2=\psi_\alpha^2$ almost surely.
Summing the variances and using $\E[\psi_\alpha]=\E[Y-Z]$ proves
Eq.~\eqref{eq:bin-variance-sum}. Its other terms do not depend on
$\alpha$, and each bin-moment estimate is unbiased.

For the second part, zero-mass bins contribute zero almost surely. Let
$s(z)\coloneqq\sum_{\ell=1}^M\operatorname{sgn}(r_\ell)\one(z\in B_\ell)$
and $H_n\coloneqq n^{-1}\sum_{i=1}^n\psi_\alpha(O_i)s(Z_i)$.
Then $s(Z)^2=1$ almost surely,
$\E[H_n]=\ECE_\calB(f)$, and
$\Var(H_n)=(\E[\psi_\alpha^2]-\ECE_\calB(f)^2)/n$.
To control sign changes, let $d\coloneqq\min_{\ell:P(Z\in B_\ell)>0}|r_\ell|>0$ and
$B_\rho\coloneqq\max_{j\in[K]}|\delta_j|^{-1}$.
We have $|\psi_\alpha|\le B_\rho$, and Hoeffding's inequality gives
$P(\wh E_\alpha\ne H_n) \le 2M\exp\left(-\frac{nd^2}{2B_\rho^2}\right).$
Indeed, the two statistics are equal whenever no estimated bin moment
has the opposite sign to its population moment.
Also $|\wh E_\alpha-H_n|\le2B_\rho$, so
$\E[(\wh E_\alpha-H_n)^2]$ decays exponentially.
It follows that
$n\bigl|\Var(\wh E_\alpha)-\Var(H_n)\bigr|\to0$.
The identity
$\E[\psi_\alpha^2]=V_{\mathrm{res}}(\alpha)+(\E[Y-Z])^2$
proves Eq.~\eqref{eq:bin-ece-asymptotic-variance}.
Applying Lemma~\ref{lem:variance-optimal-alpha} completes the argument.
\end{proof}

\begin{example}[Residual and finite-sample ECE variances can favor different weights]
\label{ex:ece-variance-not-residual}
Let $K=2$, $Z=1/2$, $Y\sim\operatorname{Bernoulli}(1/2)$, and suppose
the two annotators are conditionally independent with identical known
channels $a_1=a_2=1/4$, $b_1=b_2=3/4$.
Each $\psi_j=2A_j-1$ takes values $-1,1$ with equal probabilities.
For $n=1$ and one bin, choosing $\alpha=(1,0)$ gives
$\Var(\psi_\alpha)=1$ but $\wh E_\alpha=|\psi_1|=1$ almost surely,
so $\Var(\wh E_\alpha)=0$.
The IV weights are $\alpha=(1/2,1/2)$, with
$\psi_\alpha=A_1+A_2-1$. They reduce the residual variance to
$\Var(\psi_\alpha)=5/8$, but $\wh E_\alpha$ is Bernoulli with success
probability $5/8$ and hence $\Var(\wh E_\alpha)=15/64>0$.
Both residual-moment estimates are unbiased for zero.
\end{example}
The asymptotic statement above excludes zero population bin moments and
does not cover an increasing number of bins or estimated channels. Signed bin
moments are unbiased with known channels, whereas their absolute values have
nonnegative finite-sample bias by Jensen's inequality. Small channel gaps can
amplify this effect through larger corrected-label variance. IV weighting
reduces signed-moment variance but does not remove this bias, so the result is
not a finite-sample optimality claim for the absolute-value statistic.

\section{Channel Identification and Estimation}
\label{v03:identification-proofs}
Identification asks whether two parameter values can produce the same
distribution of available data. Estimation asks how accurately the uniquely
determined parameters can be recovered from a finite sample. We address
these questions in that order. The channel-error bounds can then be
inserted into the calibration bounds of Appendix~\ref{v03:main-proofs}.
Within each model, $\theta$ denotes a candidate parameter vector,
$\theta_0$ its true value, and $\wh\theta$ an estimate based on $\calD_\ch$.
The coordinates of $\wh\theta$ corresponding to $a_j,b_j$ are the
channel estimates $\wh a_j,\wh b_j$ used in Section~\ref{v03:estimation}.
\subsection{Non-identification without sufficient information}
\label{app:nonidentification}
\begin{proposition}[One annotator]
\label{thm:one-annotator}
Suppose $K=1$ and $q(z)\coloneqq P(A_1=1\mid Z=z)$ satisfies
\begin{equation}\textstyle
 0<\essinf q(Z)\le\esssup q(Z)<1.
 \label{eq:essential-range}
\end{equation}
There is a continuum of positively oriented class-conditional channels
and true-label distributions inducing the same distribution of $(Z,A_1)$.
The population true-label calibration error is not generally identifiable
in this model.
\end{proposition}
\begin{proof}
Keep $P_Z$ fixed. For any pair satisfying
$0\le a'<\essinf q(Z)\le\esssup q(Z)<b'\le1$, define
$m'(z)\coloneqq(q(z)-a')/(b'-a')$ outside a $P_Z$-null set and extend it
measurably to that set, for example by setting it to $1/2$.
Then $0<m'(Z)<1$ almost surely.
Generate $Y'\mid Z=z\sim\Ber(m'(z))$ and
$A_1\mid(Y'=0,Z)\sim\Ber(a')$,
$A_1\mid(Y'=1,Z)\sim\Ber(b')$.
This construction has $P(A_1=1\mid Z)=a'+(b'-a')m'(Z)=q(Z)$,
so the distribution of the available data is unchanged.
The condition in Eq.~\eqref{eq:essential-range} provides a continuum of such pairs.

To show that calibration itself can change, take $Z\sim\mathrm{Unif}[0,1]$.
The models
$(a,b,m(z))=(0.2,0.8,0.25+0.5z),\quad (a',b',m'(z))=(0.1,0.9,0.3125+0.375z)$
both give $q(z)=0.35+0.3z$.
Their calibration errors are respectively $1/8$ and $5/32$.
Thus the same available-data distribution need not determine $\ECE(f)$.
\end{proof}
The essential bounds in Eq.~\eqref{eq:essential-range} ensure that the
alternative channels remain valid probabilities. The example establishes
non-identification of ECE in general; particular equivalent models can
still have equal calibration errors.

\paragraph{Two annotators after discarding scores.}
The following family produces identical marginal annotation distributions
with different noise probabilities.
Set $\mu=1/2$, choose $d\in(0.24,1)$, and define
$\delta_1\coloneqq d$, $\delta_2\coloneqq0.24/d$,
$a_j\coloneqq(1-\delta_j)/2$, $b_j\coloneqq(1+\delta_j)/2$ for $j\in[2]$.
Generate the two annotations independently conditional on $Y$.
For every such $d$, both annotation means equal $1/2$ and their covariance
is $\mu(1-\mu)\delta_1\delta_2=0.06$.
Hence the $(0,0),(0,1),(1,0),(1,1)$ probabilities are always
$(0.31,0.19,0.19,0.31)$, although the channels vary continuously.
All channels are positively oriented and interior.
The marginal two-annotator model is therefore not generally identified.
Longer optimization cannot recover information discarded with the scores.

\subsection{Two annotators in two score groups}
\label{app:k2}
Let $G\coloneqq\one(Z\ge t)$, where $t\in(0,1)$ is fixed independently of
the channel sample, or is chosen on an independent pilot sample.
All statements can be conditioned on that pilot sample.
For $g\in\{0,1\}$ and $j\in[2]$, assume
\begin{equation}\textstyle
 P(G=g)>0,\quad 0<\mu_g<1,\quad \mu_0\ne\mu_1,
 \quad 0\le a_j<b_j\le1,
 \label{eq:k2-conditions}
\end{equation}
and impose Assumptions~\ref{v03:marginal} and~\ref{v03:product} with $K=2$.
The shared channels and group proportions are
$\theta\coloneqq(a_1,b_1,a_2,b_2,\mu_0,\mu_1)$.
Define $q_{1,g},q_{2,g},q_{12,g}$ as in
Section~\ref{sec:two-group-method}; their model equations are
Eq.~\eqref{v03:moments} for $g\in\{0,1\}$. Let
$\Gamma(\theta)\coloneqq(q_{1,0},q_{2,0},q_{12,0},q_{1,1},q_{2,1},q_{12,1})$.

\begin{theorem}[Closed-form two-group identification]
\label{thm:k2-identification}
Under Assumptions~\ref{v03:marginal} and~\ref{v03:product} and the conditions in
Eq.~\eqref{eq:k2-conditions}, let
\begin{align}
 c_g&\coloneqq q_{12,g}-q_{1,g}q_{2,g},\qquad
 \beta\coloneqq\frac{q_{2,1}-q_{2,0}}{q_{1,1}-q_{1,0}},\qquad
 h_g\coloneqq\frac{c_g}{\beta}+q_{1,g}^2,\label{eq:k2-intermediate}\\
 u&\coloneqq\frac{h_1-h_0}{q_{1,1}-q_{1,0}},\qquad
 v\coloneqq u q_{1,0}-h_0,\qquad \tau\coloneqq q_{2,0}-\beta q_{1,0}.
 \label{eq:k2-quadratic}
\end{align}
Then the channels and proportions are uniquely recovered by
\begin{align}
 a_1&=\frac{u-\sqrt{u^2-4v}}2,\qquad
 b_1=\frac{u+\sqrt{u^2-4v}}2,\notag\\
 a_2&=\tau+\beta a_1,\qquad b_2=\tau+\beta b_1,\qquad
 \mu_g=\frac{q_{1,g}-a_1}{b_1-a_1},\quad g\in\{0,1\}.
 \label{eq:k2-inverse}
\end{align}
The Jacobian determinant is
\begin{equation}\textstyle
 \det D\Gamma(\theta)=\delta_1^2\delta_2^2(\mu_1-\mu_0)^2.
 \label{eq:k2-determinant}
\end{equation}
\end{theorem}
\begin{proof}
The group-moment equations follow by conditioning on $Y$ and using the
product assumption. In particular,
$q_{1,1}-q_{1,0}=\delta_1(\mu_1-\mu_0)\ne0$ and
$q_{2,1}-q_{2,0}=\delta_2(\mu_1-\mu_0)$, so
$\beta=\delta_2/\delta_1>0$.
Conditional independence gives
$c_g=\delta_1\delta_2\mu_g(1-\mu_g)$.
Since $q_{1,g}-a_1=\delta_1\mu_g$ and
$b_1-q_{1,g}=\delta_1(1-\mu_g)$,
$\frac{c_g}{\beta}=(q_{1,g}-a_1)(b_1-q_{1,g}),\quad h_g=(a_1+b_1)q_{1,g}-a_1b_1.$
The two distinct values of $q_{1,g}$ identify the slope $a_1+b_1$ and the
negative intercept $-a_1b_1$. The polynomial
$q^2-uq+v=(q-a_1)(q-b_1)$ therefore has the two ordered roots in
Eq.~\eqref{eq:k2-inverse}. Positive orientation fixes the ordering.
Also $q_{2,g}-\beta q_{1,g}=a_2-\beta a_1$, proving the remaining channel
formulas. Finally $\mu_g=(q_{1,g}-a_1)/\delta_1$.
This proves identification and Theorem~\ref{v03:k2}.

For completeness, replace the input coordinates by
$(a_1,\delta_1,a_2,\delta_2,\mu_0,\mu_1)$ and the output coordinates
by $(q_{1,0},q_{2,0},c_0,q_{1,1},q_{2,1},c_1)$.
Both changes have determinant one. The resulting Jacobian is
\begin{equation}
 \begin{pmatrix}
 1&\mu_0&0&0&\delta_1&0\\
 0&0&1&\mu_0&\delta_2&0\\
 0&\delta_2\mu_0(1-\mu_0)&0&\delta_1\mu_0(1-\mu_0)&\delta_1\delta_2(1-2\mu_0)&0\\
 1&\mu_1&0&0&0&\delta_1\\
 0&0&1&\mu_1&0&\delta_2\\
 0&\delta_2\mu_1(1-\mu_1)&0&\delta_1\mu_1(1-\mu_1)&0&\delta_1\delta_2(1-2\mu_1)
 \end{pmatrix}.
 \label{eq:k2-jacobian-matrix}
\end{equation}
Subtract the first three rows from their corresponding last rows and write
$d=\mu_1-\mu_0$ and $q=1-\mu_0-\mu_1$. Expanding along the first and third
columns leaves the negative of a $4\times4$ determinant. Subtracting
$q\delta_2$ times its second row and $q\delta_1$ times its third row from its
fourth row leaves $(0,0,\delta_1\delta_2(1-2\mu_1),-\delta_1\delta_2(1-2\mu_0))$.
Expansion gives $-\delta_1^2\delta_2^2d^2$; the preceding sign proves
Eq.~\eqref{eq:k2-determinant}.
\end{proof}

These are the classical two-test, two-population conditions
\citep{hui1980estimating}. Scores need not be calibrated, but the group class
proportions must differ. Small between-group response-mean differences or a
small minimum singular value of $D\Gamma(\wh\theta)$ diagnose weak
identification, but do not test the latent condition $\mu_0\ne\mu_1$. The
exactly identified model supplies no excess moments for testing its assumptions.

\subsection{Concentration and stable inversion}
\label{app:inversion-bound}
Let $\calD_\ch$ consist of i.i.d. examples and let $G_i$ be the group
assigned to its $i$th example by the fixed score partition.
For groups indexed by $[L]$, let
$n_g\coloneqq\sum_{i=1}^{n_\ch}\one(G_i=g)$ and
$n_{\min}\coloneqq\min_{g\in[L]}n_g$.
For positive group counts, define, for $j\in[2]$ and $g\in[L]$,
\begin{equation}\textstyle
 \wh q_{j,g}\coloneqq\frac1{n_g}\sum\nolimits_{i:G_i=g}A_{ij},
 \qquad
 \wh q_{12,g}\coloneqq\frac1{n_g}\sum\nolimits_{i:G_i=g}A_{i1}A_{i2}.
 \label{eq:empirical-group-moments}
\end{equation}
Stack $(\wh q_{1,g},\wh q_{2,g},\wh q_{12,g})$ in increasing order of $g$
to obtain $\wh\Gamma_L$. Its population counterpart is $\Gamma_L(\theta_0)$.

\begin{lemma}[Group-moment concentration]
\label{thm:group-moment-concentration}
For $\eta\in(0,1)$, conditional on positive group counts,
\begin{equation}
 P\left(\|\wh\Gamma_L-\Gamma_L(\theta_0)\|_\infty
 >\sqrt{\frac{\log(6L/\eta)}{2n_{\min\nolimits}}}
 \,\middle|\,(n_g)_{g\in[L]}\right)\le\eta.
 \label{eq:moment-concentration}
\end{equation}
\end{lemma}
\begin{proof}
Conditional on the group assignments, each coordinate is an average of
$n_g$ independent variables in $[0,1]$ from that group's distribution.
Hoeffding's inequality bounds its two-sided deviation by
$2\exp(-2n_gt^2)\le2\exp(-2n_{\min}t^2)$.
The bound does not depend on which indices form each group, so it also
holds conditional only on group counts. A union bound over $3L$
coordinates and the stated choice of $t$ prove the result.
Independence between coordinates is not needed.
\end{proof}

For two groups, use indices $g\in\{0,1\}$ as in Appendix~\ref{app:k2}
and write $\wh\Gamma$ for the resulting six sample moments.
Let $H$ denote the inverse given by Eq.~\eqref{eq:k2-intermediate} and Eq.~\eqref{eq:k2-inverse}.
The method-of-moments estimator is
$\wh\theta_{\mom}\coloneqq H(\wh\Gamma)$ whenever the formulas are real,
the denominators are nonzero, and the resulting probabilities are feasible
with positive gaps. Otherwise the inversion procedure reports failure.
The following theorem gives a high-probability event on which the estimate
is both defined and accurate.

\begin{theorem}[Finite-sample closed-form inversion]
\label{thm:k2-direct-finite}
Suppose the conditions of Theorem~\ref{thm:k2-identification} hold and,
for fixed positive constants $\pi_{\min},\delta_{\min},\gamma_{\min},\epsilon$,
\begin{equation}\textstyle
 P(G=g)\ge\pi_{\min\nolimits},\quad
 \delta_j\ge\delta_{\min\nolimits},\quad
 |\mu_1-\mu_0|\ge\gamma_{\min\nolimits},\quad
 a_j,b_j,\mu_g\in[\epsilon,1-\epsilon]
 \label{eq:stable-k2}
\end{equation}
for $j\in[2]$, $g\in\{0,1\}$.
There is $R_0>0$ such that the inverse $H$ in Eq.~\eqref{eq:k2-inverse}
is continuously differentiable on an open set containing the closed
Euclidean ball $N\coloneqq\overline B_2(\Gamma(\theta_0),R_0)$ and maps $N$
into feasible, positively oriented probability parameters.
Let $L_H\coloneqq\sup_{\gamma\in N}\|DH(\gamma)\|_\op<\infty$.
Conditional on positive group counts, if
\begin{equation}\textstyle
 \sqrt6\,\epsilon_n(\eta)\le R_0,
 \qquad \epsilon_n(\eta)\coloneqq\sqrt{\frac{\log(12/\eta)}{2n_{\min\nolimits}}},
 \label{eq:inverse-radius}
\end{equation}
then with probability at least $1-\eta$, the empirical inverse is feasible
and
\begin{equation}
 \|\wh\theta_{\mom}-\theta_0\|_2\le\sqrt6 L_H\epsilon_n(\eta).
 \label{eq:k2-finite-bound}
\end{equation}
\end{theorem}
\begin{proof}
At the true moment vector,
$|q_{1,1}-q_{1,0}|\ge\delta_{\min}\gamma_{\min}$,
$\beta=\delta_2/\delta_1>0$, and the discriminant equals
$\delta_1^2\ge\delta_{\min}^2$.
All denominators and the square root are therefore separated from their
singularities. The inverse is continuously differentiable on a sufficiently
small open neighborhood. By continuity and the interior assumptions,
choose a closed ball $N$ contained in that neighborhood whose image remains
feasible and positively oriented. Continuity of $DH$ on the compact ball
gives finite $L_H$.
On the event from Lemma~\ref{thm:group-moment-concentration} with $L=2$,
$\|\wh\Gamma-\Gamma(\theta_0)\|_2\le\sqrt6\epsilon_n\le R_0$.
The segment joining these two vectors is in the convex ball $N$.
The integral form of the mean-value formula gives
$H(\wh\Gamma)-H(\Gamma(\theta_0)) =\int_0^1 DH(\Gamma(\theta_0)+t(\wh\Gamma-\Gamma(\theta_0))) (\wh\Gamma-\Gamma(\theta_0))\,dt.$
Taking norms proves Eq.~\eqref{eq:k2-finite-bound}.
\end{proof}
For the asymptotic statement, extend the estimator by any fixed feasible
value when inversion fails; the finite-sample bound concerns its valid
event. Experimental failures remain reported separately.
Each individual endpoint error is at most the Euclidean bound, while
$|\wh a_j-a_j|+|\wh b_j-b_j|\le\sqrt2\|\wh\theta_{\mom}-\theta_0\|_2$.
The endpoint bound controls $\dch$ in both independent-sample and
same-sample evaluation. The sum bound is useful when a result is stated
in terms of the maximum two-endpoint error.
Since $P(G=g)\ge\pi_{\min}$,
$P(n_g<n_\ch\pi_{\min}/2)\le\exp(-n_\ch\pi_{\min}/8)$ by a
binomial multiplicative Chernoff bound. Thus positive counts and
$n_{\min}$ of order $n_\ch$ have probability tending to one, giving
$O_p(n_\ch^{-1/2})$ parameter error.
Neither $R_0$ nor $L_H$ is an automatically observable confidence bound.
They quantify local stability under the specified separation assumptions.

\subsection{Constrained GMM and optimization error}
\label{app:gmm}
We estimate the grouped parameters by generalized method of moments (GMM)~\citep{hansen1982gmm}.
Let $\calC\coloneqq\{C_g\}_{g\in[L]}$ be a partition independent of the channel
sample, where $L\ge2$ is fixed and $G=g$ exactly when $Z\in C_g$.
In this subsection write $n\coloneqq n_\ch$.
For $K=2$, define
$\theta\coloneqq(a_1,b_1,a_2,b_2,\mu_1,\ldots,\mu_L)\in\R^{L+4}$.
For each candidate parameter $\theta$, let $\Gamma_L(\theta)\in\R^{3L}$
collect $(q_{1,g},q_{2,g},q_{12,g})$ from the same three
model equations as Eq.~\eqref{v03:moments}, now for each $g\in[L]$, in
increasing order of $g$. Let $\wh\Gamma_L$ contain the corresponding
empirical averages in Eq.~\eqref{eq:empirical-group-moments}.
\begin{definition}[Constrained GMM estimator]
\label{def:gmm-main}
Let $W_n\in\R^{3L\times3L}$ be symmetric positive definite, and let $\Theta$ be compact, contain the true
parameter $\theta_0$, and impose interior probability bounds and gaps
at least $\delta_{\min}>0$. Define
\begin{equation}\textstyle
 Q_n(\theta)\coloneqq\big\|W_n^{1/2}(\wh\Gamma_L-\Gamma_L(\theta))\big\|_2^2,
 \qquad \wh\theta_\GMM\in\argmin\nolimits_{\theta\in\Theta}Q_n(\theta),
 \label{v03:gmm}
\end{equation}
choosing any measurable minimizer.
\end{definition}
We use the following explicit conditions.
\begin{assumption}[GMM estimation conditions]
\label{ass:gmm}
Assumptions~\ref{v03:marginal} and~\ref{v03:product} hold with $K=2$, and:
\begin{enumerate}[label=(\roman*)]
\item $P(G=g)\ge\pi_{\min}>0$ for all $g\in[L]$.
\item $\theta_0$ belongs to a compact set $\Theta$ on which all probability
coordinates lie in $[\epsilon,1-\epsilon]$ for some $\epsilon\in(0,1/2)$,
and $b_j-a_j\ge\delta_{\min}>0$ for $j\in[2]$.
\item There exist $g_0\ne g_1$ in $[L]$ such that
$|\mu_{g_1}-\mu_{g_0}|\ge\gamma_{\min}>0$ at the true parameter.
\item The symmetric moment-weighting matrix $W_n\in\R^{3L\times3L}$
may depend on the sample but satisfies deterministic bounds
$0<\lambda_-\le\lambda_{\min}(W_n)\le\lambda_{\max}(W_n)\le\lambda_+<\infty$
almost surely.
\end{enumerate}
\end{assumption}
Define the quantitative identification constant
\begin{equation}\textstyle
 \kappa_\Theta\coloneqq \inf\nolimits_{\theta\in\Theta\setminus\{\theta_0\}}
 \frac{\|\Gamma_L(\theta)-\Gamma_L(\theta_0)\|_2}{\|\theta-\theta_0\|_2}.
 \label{eq:kappa}
\end{equation}
Its positivity follows from the preceding conditions; it is not a separate
assumption of a well-behaved inverse.

\begin{lemma}[Quantitative identification on the compact parameter set]
\label{lem:gmm-identification}
Under Assumption~\ref{ass:gmm}, $\Gamma_L(\theta)=\Gamma_L(\theta_0)$
for $\theta\in\Theta$ implies $\theta=\theta_0$, its Jacobian at
$\theta_0$ has full column rank, and $\kappa_\Theta>0$.
\end{lemma}
\begin{proof}
Restrict the moments to the two separated groups $g_0,g_1$.
Their first-annotator means differ because
$\delta_1(\mu_{g_1}-\mu_{g_0})\ne0$.
Every parameter with these same moments must therefore also have different
group proportions. The two-group inverse identifies its shared channels
and these two proportions uniquely. All remaining proportions follow
from $\mu_g=(q_{1,g}-a_1)/\delta_1$. This proves identification of $\theta_0$
on $\Theta$.

If $D\Gamma_L(\theta_0)v=0$, restriction to these same two groups and
Eq.~\eqref{eq:k2-determinant} forces the four channel coordinates of $v$
and its two group-proportion coordinates to vanish.
For each remaining group, the derivative of $q_{1,g}$ is then
$\delta_1v_{\mu_g}=0$, so $v=0$.
Let $s_0>0$ be the smallest singular value of this full-column-rank
Jacobian. Differentiability gives a radius $r_0>0$ such that
$\|\Gamma_L(\theta)-\Gamma_L(\theta_0)\|_2 \ge(s_0/2)\|\theta-\theta_0\|_2 \quad\text{when }\|\theta-\theta_0\|_2\le r_0.$
On the remaining compact set, the numerator has a positive minimum by
identification and continuity, while the denominator is bounded above.
Combining the two regions gives $\kappa_\Theta>0$.
\end{proof}
This lemma does not provide a separation-free numerical lower bound:
$\kappa_\Theta$ may become small as gaps or group differences decrease.

\begin{theorem}[Finite-sample GMM bound]
\label{thm:gmm-finite}
Suppose Assumption~\ref{ass:gmm} holds and condition on positive group counts. For the objective $Q_n$ in Definition~\ref{def:gmm-main}, any measurable feasible estimate satisfying
\begin{equation}\textstyle
 Q_n(\wh\theta)\le\inf\nolimits_{\theta\in\Theta}Q_n(\theta)+\epsilon_{\mathrm{opt}},
 \qquad \epsilon_{\mathrm{opt}}\ge0,
 \label{eq:gmm-approximate}
\end{equation}
satisfies, for every $\eta\in(0,1)$, with probability at least $1-\eta$,
\begin{equation}
 \|\wh\theta-\theta_0\|_2
 \le \frac{2\sqrt{3L}}{\kappa_\Theta}
   \sqrt{\frac{\lambda_{\max\nolimits}(W_n)}{\lambda_{\min\nolimits}(W_n)}}
   \sqrt{\frac{\log(6L/\eta)}{2n_{\min\nolimits}}}
   +\frac{\sqrt{\epsilon_{\mathrm{opt}}}}
          {\kappa_\Theta\sqrt{\lambda_{\min\nolimits}(W_n)}}.
 \label{eq:gmm-finite}
\end{equation}
A global minimizer has $\epsilon_{\mathrm{opt}}=0$.
\end{theorem}
\begin{proof}
Because $\theta_0\in\Theta$, Eq.~\eqref{eq:gmm-approximate} and
$\sqrt{x^2+y}\le x+\sqrt y$ for $x,y\ge0$ give
$\|W_n^{1/2}(\wh\Gamma_L-\Gamma_L(\wh\theta))\|_2 \le\|W_n^{1/2}(\wh\Gamma_L-\Gamma_L(\theta_0))\|_2 +\sqrt{\epsilon_{\mathrm{opt}}}.$
The triangle inequality and the eigenvalue bounds therefore imply
$\|\Gamma_L(\wh\theta)-\Gamma_L(\theta_0)\|_2 \le2\sqrt{\lambda_{\max}(W_n)/\lambda_{\min}(W_n)} \|\wh\Gamma_L-\Gamma_L(\theta_0)\|_2 +\sqrt{\epsilon_{\mathrm{opt}}/\lambda_{\min}(W_n)}.$
Apply Eq.~\eqref{eq:kappa},
$\|v\|_2\le\sqrt{3L}\|v\|_\infty$ in $\R^{3L}$, and
Lemma~\ref{thm:group-moment-concentration}.
All inequalities hold for each realized $W_n$, so its independence from
$\wh\Gamma_L$ is not required.
\end{proof}
A stopping tolerance or a stationary point is not a certificate for
Eq.~\eqref{eq:gmm-approximate} with a specified global objective gap.
The theorem for a global minimizer must not be attributed to an arbitrary
numerical output without that additional justification.
For fixed $L$, a certified optimization gap $O_p(n_\ch^{-1})$ preserves
the root-$n_\ch$ order under the stated stability conditions.

\paragraph{Propagation to ECE.}
Let $R_\theta$ be the right-hand side of Eq.~\eqref{eq:gmm-finite}, using
channel failure probability $\eta_\ch$.
For any channel-data-dependent convex weights,
$\dch\le R_\theta/\delta_{\min}$ and $\Bn\le1/\delta_{\min}$ because
the feasible estimate has gaps at least $\delta_{\min}$.
If $r(z)\coloneqq m(z)-z$ has an $L_r$-Lipschitz version, independent evaluation
gives, with probability at least $1-\eta_\ch-\eta_\ev$ conditional on
positive group counts,
\begin{equation}
 \abs{\wh\ECE^{\nc}_\calB(f,\calD_\ev)-\ECE(f)}
 \le L_rh_\calB+
 \frac1{\delta_{\min\nolimits}}\sqrt{\frac{2(M\log2+\log(2/\eta_\ev))}{n_\ev}}
 +\frac{R_\theta}{\delta_{\min\nolimits}}.
 \label{eq:gmm-ece}
\end{equation}
For a global minimizer, the last term is exactly
\begin{equation}
 \frac{2\sqrt{3L}}{\kappa_\Theta\delta_{\min\nolimits}}
 \sqrt{\frac{\lambda_{\max\nolimits}(W_n)}{\lambda_{\min\nolimits}(W_n)}}
 \sqrt{\frac{\log(6L/\eta_\ch)}{2n_{\min\nolimits}}}.
 \label{eq:gmm-ece-exact}
\end{equation}
This independent-sample propagation has one inverse gap factor, in
addition to the conditioning already present in $\kappa_\Theta$.
For same-sample evaluation, the same bounds
$\dch\le R_\theta/\delta_{\min}$ and $\Bn\le1/\delta_{\min}$
can be inserted directly into Corollary~\ref{cor:same-sample-ece}.
No additional channel-radius restriction is needed, since the feasible
GMM set already imposes the lower gap bound. The logarithmic term then
uses $\log(12/\eta_\ev)$ because $K=2$.
The same-sample concentration event is unconditional. Combine it with the
group-count event as below; conditioning it on successful channel fitting
is not justified by the uniform bound alone.

\paragraph{Accounting for random group counts.}
The preceding bounds condition on positive counts, rather than treating
these counts as fixed design quantities. For an unconditional statement,
the multiplicative binomial bound gives
$P(n_{\min}<n_\ch\pi_{\min}/2) \le L\exp(-n_\ch\pi_{\min}/8).$
On its complement, replace $n_{\min}$ in Eq.~\eqref{eq:gmm-finite}
by $n_\ch\pi_{\min}/2$. Combining with the channel and evaluation
events gives probability at least
$1-\eta_\ch-\eta_\ev-L\exp(-n_\ch\pi_{\min}/8)$ for
Eq.~\eqref{eq:gmm-ece} with that enlarged radius.
A procedure returning failure when a group is empty is thus covered on
an explicit high-probability success event. This does not assign zero
estimation loss to failed runs.

\subsection{Marginal identification with three informative annotators}
\label{app:ds-identification}
Let $K\ge3$, $\mu\coloneqq P(Y=1)\in(0,1)$, and assume
Assumptions~\ref{v03:marginal} and~\ref{v03:product}. Its probabilities are $p_\theta(s)$ in Eq.~\eqref{v03:ds}.
The following moment argument gives the binary identification result
under the stated nondegeneracy conditions; it is consistent with the
latent-class identification framework of \citet{allman2009identifiability}.

For the identities behind the argument in Section~\ref{sec:three-annotator-method},
relabel an informative triple as $1,2,3$, define
$q_j\coloneqq\E[A_j]$ for $j\in[K]$ and
$C_{jk}\coloneqq\Cov(A_j,A_k)$ for distinct $j,k\in[K]$, and let
$T\coloneqq\E[(A_1-q_1)(A_2-q_2)(A_3-q_3)]$.
Then
\begin{equation}\textstyle
 C_{jk}=\mu(1-\mu)\delta_j\delta_k,\qquad
 T=\mu(1-\mu)(1-2\mu)\delta_1\delta_2\delta_3.
 \label{eq:ds-identification-moments-main}
\end{equation}
Combining them eliminates the channel gaps:
\begin{equation}\textstyle
 \frac{T^2}{C_{12}C_{13}C_{23}}
 =\frac{(1-2\mu)^2}{\mu(1-\mu)}
 =\frac1{\mu(1-\mu)}-4.
 \label{eq:ds-cancellation-main}
\end{equation}
The proposition and its proof justify these identities and recover the remaining parameters.

\begin{proposition}[Binary marginal identification]
\label{prop:ds-identification}
\label{prop:ds-main}
Under the conditions above, suppose three specified annotators are
informative and the sign of one of their gaps is known. Then the marginal distribution of $A_{1:K}$
identifies $\mu$ and every $a_j,b_j$, up to the global permutation before
orientation. Other modeled annotators may be uninformative.
\end{proposition}
\begin{proof}
Relabel the specified informative triple as $1,2,3$ and first take the known
anchor sign to be $\delta_1>0$. For a negative anchor sign, exchange the
two class labels for this calculation and undo that exchange afterwards. Put
$q_j\coloneqq\E[A_j]$,
$C_{jk}\coloneqq\E[(A_j-q_j)(A_k-q_k)]$ for distinct $j,k\in[K]$, and
$T\coloneqq\E[(A_1-q_1)(A_2-q_2)(A_3-q_3)]$.
To obtain the moment identities, write
$\xi_j\coloneqq A_j-a_j-\delta_jY$.
The two assumptions imply conditional independence given $Y$ and
$\E[\xi_j\mid Y]=0$. Also $q_j=a_j+\mu\delta_j$, so
$A_j-q_j=\delta_j(Y-\mu)+\xi_j$.
In products involving distinct annotators, terms containing any $\xi_j$
have zero conditional expectation given $Y$. Using
$\E[(Y-\mu)^2]=\mu(1-\mu)$ and
$\E[(Y-\mu)^3]=\mu(1-\mu)(1-2\mu)$ therefore gives
\begin{equation}\textstyle
 C_{jk}=\mu(1-\mu)\delta_j\delta_k,
 \qquad
 T=\mu(1-\mu)(1-2\mu)\delta_1\delta_2\delta_3.
 \label{eq:ds-moments}
\end{equation}
For the triple,
$C_{12}C_{13}C_{23} =[\mu(1-\mu)]^3\delta_1^2\delta_2^2\delta_3^2>0.$
Dividing the squared third moment by this expression gives
Eq.~\eqref{eq:ds-cancellation-main}; the channel gaps cancel exactly.
To fix the branch left by that squared ratio, use the known anchor sign
and the sign of the third moment.
Let $s_\delta\coloneqq\operatorname{sign}(C_{12}C_{13})$ and
$t\coloneqq s_\delta T/\sqrt{C_{12}C_{13}C_{23}}$.
Because $\delta_1>0$, $s_\delta$ is the sign of
$\delta_1\delta_2\delta_3$, and hence
$t=(1-2\mu)/\sqrt{\mu(1-\mu)}$.
Solving gives
\begin{equation}
 \mu=\frac12\left(1-\frac{t}{\sqrt{t^2+4}}\right),
 \qquad
 \delta_1=\sqrt{\frac{C_{12}C_{13}}{C_{23}\mu(1-\mu)}}.
 \label{eq:ds-inverse}
\end{equation}
When $\mu=1/2$, $T=0$ and $t=0$, so the same formulas recover
$\mu=1/2$ and the oriented gaps. No nonzero-third-moment condition is
needed.
For $j\in[K]\setminus\{1\}$,
$\delta_j=C_{1j}/(\mu(1-\mu)\delta_1)$, including zero gaps.
Finally $a_j=q_j-\mu\delta_j$ and $b_j=a_j+\delta_j$.
All formulas are determined by the marginal distribution and the chosen
orientation. Reversing the orientation applies the one global label
permutation, completing the proof.
\end{proof}
An uninformative channel can be estimated as part of this likelihood
model but cannot be inverted as a known channel. In applying the
known-channel correction or variance formulas, use only a specified
informative subset as described in Appendix~\ref{v03:notation}.
The estimated-channel bound of Theorem~\ref{v03:bound} instead requires
positive estimated gaps and records any large channel error explicitly;
it does not make an uninformative annotator informative.

\subsection{Regular likelihood estimators and grouped models}
\label{app:likelihood-rates}
Write $A=A_{1:K}$ in the likelihood formulas.
For marginal estimation, let
$\ell_\theta(A)\coloneqq\log p_\theta(A)$ and
$\theta\coloneqq(\mu,a_1,b_1,\ldots,a_K,b_K)$.
For grouped estimation with fixed $L$, set, for $g\in[L]$ and
$s\in\{0,1\}^K$,
\begin{align}
 p_\theta(s\mid g)&\coloneqq(1-\mu_g)\prod\nolimits_{j=1}^{K}a_j^{s_j}(1-a_j)^{1-s_j}
        +\mu_g\prod\nolimits_{j=1}^{K}b_j^{s_j}(1-b_j)^{1-s_j},
 \label{eq:grouped-likelihood}\\
 \ell_\theta(G,A)&\coloneqq\log p_\theta(A\mid G),\qquad
 \theta\coloneqq((a_j,b_j)_{j\in[K]},(\mu_g)_{g\in[L]}).
 \label{eq:grouped-parameter}
\end{align}
The group probabilities $\pi_g\coloneqq P(G=g)>0$ do not depend on these
parameters. The joint distribution of $(G,A)$ is
$\pi_g p_\theta(s\mid g)$, so maximizing the conditional likelihood is
appropriate without modeling the score density inside a group.

\begin{theorem}[Consistency and rate of likelihood estimation]
\label{thm:kge3-regular-mle}
Consider either of these correctly specified finite-alphabet models with
fixed $K$ and, for the grouped model, fixed $L$.
Let $\Theta$ be compact, with all channel and class-proportion coordinates
in $[\epsilon,1-\epsilon]$ for some $\epsilon\in(0,1/2)$.
Assume the prescribed orientation identifies the true parameter
$\theta_0$ on $\Theta$. Let $V_1,\ldots,V_{n_\ch}$ be i.i.d. observations
from that model at $\theta_0$, with $A_i=A_{i,1:K}$. Write $n\coloneqq n_\ch$
and define the normalized log likelihood
$L_n(\theta)\coloneqq n_\ch^{-1}\sum_{i=1}^{n_\ch}\ell_\theta(V_i)$,
where $V_i=A_i$ or $(G_i,A_i)$ as appropriate.
Let $\wh\theta\in\Theta$ be a measurable estimate satisfying
\begin{equation}
 L_n(\wh\theta)\ge\sup\nolimits_{\theta\in\Theta}L_n(\theta)
                       -\epsilon_{\mathrm{opt}},
 \qquad \epsilon_{\mathrm{opt}}\ge0,\quad
 \epsilon_{\mathrm{opt}}=o_p(1).
 \label{eq:mle-objective-gap}
\end{equation}
Then $\wh\theta\to\theta_0$ in probability.
If, additionally, $\theta_0$ is interior to $\Theta$ and
$I(\theta_0)\coloneqq\E_{\theta_0}[ \nabla\ell_{\theta_0}(V)\nabla\ell_{\theta_0}(V)^{\mathsf T}]$
is positive definite, let $e_{n_\ch}>0$ be a deterministic sequence tending to zero
such that $\epsilon_{\mathrm{opt}}=O_p(e_{n_\ch})$. Then
\begin{equation}
 \|\wh\theta-\theta_0\|_2
 =O_p\left(n_\ch^{-1/2}+\sqrt{e_{n_\ch}}\right).
 \label{eq:mle-rate}
\end{equation}
In particular, exact global maximizers are consistent and have
root-$n_\ch$ error. The same order holds when
$\epsilon_{\mathrm{opt}}=O_p(n_\ch^{-1})$.
\end{theorem}
\begin{proof}
\emph{Uniform convergence.}
All model probabilities are uniformly positive on $\Theta$ because the
alphabet is finite and all coordinates are bounded away from zero and one.
The log probabilities are therefore bounded and continuous on $\Theta$.
Writing $L(\theta)\coloneqq\E_{\theta_0}\ell_\theta(V)$, convergence of the
finite vector of empirical cell frequencies gives
$\sup_{\theta\in\Theta}|L_n(\theta)-L(\theta)|\to0$ in probability.
\emph{Identification of the population maximum.}
For the marginal model,
$L(\theta_0)-L(\theta)=D_{\mathrm{KL}}(p_{\theta_0}\|p_\theta)$.
For the grouped model, this difference is
$\sum_{g=1}^{L}\pi_gD_{\mathrm{KL}}
(p_{\theta_0}(\cdot\mid g)\|p_\theta(\cdot\mid g))$.
Here $D_{\mathrm{KL}}(p\|q)\coloneqq\sum_s p(s)\log(p(s)/q(s))$.
Positivity of the group probabilities and identification give a unique
population maximum at $\theta_0$.
Compactness then implies strict population separation outside every fixed
neighborhood of $\theta_0$. Uniform convergence and
Eq.~\eqref{eq:mle-objective-gap} prove consistency.
Continuity also ensures existence of an exact global maximizer.

\emph{Local error bound.}
For the rate, choose a closed convex neighborhood $N$ of $\theta_0$
inside $\Theta$. The first two log-likelihood derivatives are uniformly
bounded and continuous there. Differentiating the finite sum of
probabilities gives
$\E_{\theta_0}\nabla\ell_{\theta_0}=0$ and
$\nabla^2L(\theta_0)=-I(\theta_0)$.
Hence $\|\nabla L_n(\theta_0)\|_2=O_p(n_\ch^{-1/2})$.
After reducing $N$, uniform convergence of its Hessians implies, with
probability tending to one,
$\nabla^2L_n(\theta)\preceq-\lambda I_d/2$ for every $\theta\in N$,
where $\lambda\coloneqq\lambda_{\min}(I(\theta_0))>0$, $d$ is the parameter
dimension, and $I_d$ is the identity matrix.
Consistency places $\wh\theta$ in $N$ with probability tending to one.
The Taylor formula and approximate optimality then give, for
$r\coloneqq\|\wh\theta-\theta_0\|_2$,
$\lambda r^2/4\le\|\nabla L_n(\theta_0)\|_2r +\epsilon_{\mathrm{opt}}.$
Solving this scalar inequality yields
$r\le4\|\nabla L_n(\theta_0)\|_2/\lambda
+2\sqrt{\epsilon_{\mathrm{opt}}/\lambda}$ on this event.
Together with the assumed rate of the optimization gap, this proves
Eq.~\eqref{eq:mle-rate}.
\end{proof}
A separate local statement is useful for comparing this theorem with
standard likelihood theory \citep{vanderVaart1998}.
Any selected \emph{consistent interior stationary} solution also has
root-$n_\ch$ error: its empirical score equation is
\begin{equation}\textstyle
 0=\nabla L_n(\theta_0)+H_n(\wh\theta-\theta_0),\qquad
 H_n\coloneqq\int_0^1\nabla^2L_n(\theta_0+t(\wh\theta-\theta_0))\,dt.
 \label{eq:mle-expansion}
\end{equation}
Consistency and uniform Hessian convergence give
$H_n\to-I(\theta_0)$, and the empirical score is $O_p(n_\ch^{-1/2})$.
Unlike global likelihood maximization on the identified compact set,
stationarity alone does not establish the consistency used in this local
argument. The tolerance in Eq.~\eqref{eq:mle-objective-gap} is a gap in
\emph{mean} log likelihood relative to the global optimum, not the change
between consecutive EM iterations.

\begin{corollary}[Fixed-group shared-channel likelihood]
\label{cor:grouped-likelihood}
For the grouped model, suppose $\calC$ is fixed independently of the
channel sample, $L$ is fixed, and every group has positive probability.
Sufficient identification conditions are either:
(a) an informative pair of positively oriented annotators and two groups
with different interior true-label class proportions; or
(b) three informative annotators, interior group proportions, and a common
label orientation.
On an oriented compact parameter set satisfying
Theorem~\ref{thm:kge3-regular-mle}, global grouped likelihood maximizers
are consistent. With interior truth and nonsingular Fisher information,
their error is $O_p(n_\ch^{-1/2})$, and approximate global maximizers
obey Eq.~\eqref{eq:mle-rate} under its stated optimization-error condition.
This includes the fixed-$L$ statistical model fitted by Joint score-bin EM.
\end{corollary}
\begin{proof}
Under (a), the two-group inverse identifies the shared pair and its group
proportions. Other group proportions follow from either identified
annotator's mean. For an additional annotator, its means in the two
separated groups identify its two endpoints by a two-equation linear
system. Under (b), the marginal triple identifies the shared channels and
orientation by Proposition~\ref{prop:ds-identification}; group means then
identify every $\mu_g$. The compact-set and interior conditions and nonsingular Fisher
information remain explicit hypotheses. Theorem~\ref{thm:kge3-regular-mle} applied
to the finite alphabet $[L]\times\{0,1\}^K$ proves the rate.
\end{proof}
This corollary is not a growing-dimensional result. In particular, taking
$L=M\asymp n_\ev^{1/3}$ does not follow from the fixed-$L$ argument.
Increasing the number of evaluation bins reduces binning error, whereas
increasing $L$ changes the dimension of the channel-estimation problem.
Likewise, a partition constructed from the same channel sample is outside
the fixed-partition likelihood argument without additional analysis.

\begin{remark}[What EM establishes]
\label{rem:em-rate}
Exact E- and M-steps do not decrease the likelihood
\citep{dempster1977maximum,wu1983convergence}.
Neither monotonicity nor a finite number of initializations guarantees the
global objective condition or the consistent local solution required above.
A converged admissible fit is a numerical output, not a proof of global
optimality or of consistency. A numerical fit at a boundary can also fall
outside the regular interior theory.

\end{remark}

\subsection{Model checks and extensions}
\label{app:model-checks}
For two annotators and $L>2$, $3L$ moments versus $L+4$ parameters give
$2L-4$ restrictions at regular points. A marginal model has $2^K-1$ observable
probabilities and $2K+1$ parameters; $K>3$ permits additional model checks.
Discrepancy tests need their own regularity and calibration, especially near
boundaries or weak identification. Group-specific channels require informative
overlap and a common orientation in every group; effective sizes are the group
counts. Missing answers require a missingness model and identifiable overlap.
No guarantee for arbitrary selective missingness, dependence, or
instance-dependent channels is asserted.

\section{Recalibration: Statements and Proofs}
\label{v03:recalibration-proofs}
This appendix proves clean-risk guarantees for corrected-loss fitting, calibration guarantees for histogram fitting, and evaluation bounds for the fitted predictor.
\subsection{Corrected binary losses}
\label{app:corrected-loss}
Let $\ell:(0,1)\times\{0,1\}\to\R$ be a binary loss and write
$\ell_y(p)\coloneqq\ell(p,y)$ for $y\in\{0,1\}$.
\begin{definition}[Corrected binary loss]
\label{def:general-loss}
For fixed true channels with $\delta_j\ne0$ for every $j\in[K]$, and
fixed $\alpha\in\Delta_K$, define
\begin{equation}\textstyle
 \ell^{\mathrm{corr}}_{\rho,\alpha}(p,A_{1:K})
 \coloneqq\ell_0(p)+U_\alpha(A_{1:K})(\ell_1(p)-\ell_0(p)).
 \label{eq:corrected-loss-general}
\end{equation}
Its estimated-channel version, denoted by
$\ell^{\mathrm{corr}}_{\wh\rho}$, replaces $U_\alpha$ by $\wh U$ and
therefore uses both the fitted channels and the fitted convex weights.
\end{definition}
For a fixed measurable $g:[0,1]\to(0,1)$, write
$R_\ell(g)\coloneqq\E[\ell(g(Z),Y)]$.
Every map used in the following risk identities is assumed to satisfy
$\E[|\ell_0(g(Z))|+|\ell_1(g(Z))|]<\infty$; for a random class this
condition is imposed after fixing the independent fitting data.
This excludes undefined differences of infinite risks.

\begin{theorem}[Corrected risk and channel error]
\label{thm:unbiased-corrected-loss}
Suppose Assumption~\ref{v03:marginal} holds.
For known channels with nonzero gaps and fixed $\alpha\in\Delta_K$,
\begin{equation}\textstyle
 \E[\ell^{\mathrm{corr}}_{\rho,\alpha}(g(Z),A_{1:K})]=R_\ell(g)
 \label{eq:corrected-risk}
\end{equation}
for every fixed measurable $g$ for which these expectations exist.
For estimated channels satisfying the conditions of
Lemma~\ref{lem:mean-error}, no nonzero-true-gap condition is needed for
the following comparison. If a class $\calG$ satisfies
\begin{equation}\textstyle
 C_{\ell,\calG}\coloneqq \sup\nolimits_{g\in\calG}\sup\nolimits_{z\in[0,1]}
 |\ell_1(g(z))-\ell_0(g(z))|<\infty,
 \label{eq:loss-difference-bound}
\end{equation}
then risks evaluated on an independent example, conditional on the
channel-fitting data, satisfy
\begin{equation}\textstyle
 \sup\nolimits_{g\in\calG}
 \left|\E[\ell^{\mathrm{corr}}_{\wh\rho}(g(Z),A_{1:K})\mid\calD_\ch]
         -R_\ell(g)\right|
 \le C_{\ell,\calG}\dch.
 \label{eq:mean-risk-error}
\end{equation}
The class may be fixed conditional on $\calD_\ch$.
\end{theorem}
\begin{proof}
A binary loss satisfies
$\ell(p,y)=\ell_0(p)+y(\ell_1(p)-\ell_0(p))$.
Conditioning Eq.~\eqref{eq:corrected-loss-general} on $(Y,Z)$ and using
$\E[U_\alpha\mid Y,Z]=Y$ proves Eq.~\eqref{eq:corrected-risk}.
For estimated channels, conditioning on $(Z,\calD_\ch)$ shows that the
risk difference is
$\E[b_\ch(Z)(\ell_1(g(Z))-\ell_0(g(Z)))\mid\calD_\ch].$
Lemma~\ref{lem:mean-error} bounds this uniformly by
$C_{\ell,\calG}\dch$.
\end{proof}
\subsection{Parametric logarithmic-loss fitting}
\label{app:parametric-recalibration}
Fix an input clipping level $\epsilon\in(0,1/2)$ and write
$\bar z\coloneqq\Pi_{[\epsilon,1-\epsilon]}(z)$ and
$s(z)\coloneqq\log(\bar z/(1-\bar z))$.
Let $\sigma(t)\coloneqq(1+e^{-t})^{-1}$.
The two families are
\begin{align}
 g_T(z)&\coloneqq\sigma(s(z)/T),
 &&T\in[T_{\min\nolimits},T_{\max\nolimits}],\quad0<T_{\min\nolimits}<T_{\max\nolimits}<\infty,
 \label{eq:temperature-family}\\
 g_{u,v}(z)&\coloneqq\sigma(us(z)+v),
 &&(u,v)\in\Theta_P,\quad \Theta_P\subset\R^2\text{ nonempty and compact}.
 \label{eq:platt-family}
\end{align}
Compactness and bounded $s$ imply a common output interval
$[\epsilon_p,1-\epsilon_p]$ for some $\epsilon_p\in(0,1/2)$.
The logarithmic loss is
$\ell_{\log}(p,y)\coloneqq-y\log p-(1-y)\log(1-p)$.
Its corrected version is
\begin{equation}\textstyle
 \ell^{\mathrm{corr}}_{\log,\wh\rho}(p,A_{1:K})
 \coloneqq-\wh U\log p-(1-\wh U)\log(1-p).
 \label{eq:corrected-logloss}
\end{equation}
It may be negative for individual examples. Bounded probabilities are
therefore substantive conditions, not only numerical conventions. In code we
evaluate $\log(1+\exp\eta)-\wh U_i\eta$ directly with a stable softplus,
rather than using a binary-cross-entropy routine that restricts targets to
$[0,1]$.

\begin{remark}[Convexity and boundedness]
\label{rem:platt-convexity}
For $x_i=(s(Z_i),1)^{\mathsf T}$ and $\vartheta=(u,v)^{\mathsf T}$, the
corrected loss is $\log(1+e^{x_i^{\mathsf T}\vartheta})-\wh U_i x_i^{\mathsf T}\vartheta$.
Its Hessian is $\sigma(x_i^{\mathsf T}\vartheta)(1-\sigma(x_i^{\mathsf T}\vartheta))
x_ix_i^{\mathsf T}\succeq0$, regardless of $\wh U_i$.
The same argument applies to inverse temperature. The objective is convex on a
convex parameter set but can be unbounded below without the stated compactness.
\end{remark}

Conditional on $\calD_\ch$, let
$\wh\calF_{\calG}\coloneqq\{O\mapsto\ell^{\mathrm{corr}}_{\log,\wh\rho}
(g(Z),A_{1:K}):g\in\calG\}$ and define its expected absolute
Rademacher complexity by
\begin{equation}\textstyle
 \Rad_n(\wh\calF_{\calG}\mid\calD_\ch)
 \coloneqq\E_{\calD,\varepsilon}\left[
 \sup\nolimits_{h\in\wh\calF_{\calG}}
 \left|\frac1n\sum\nolimits_{i=1}^{n}\varepsilon_i h(O_i)\right|
 \,\middle|\,\calD_\ch\right],
 \label{eq:rademacher-definition}
\end{equation}
where $\calD$ consists of $n$ i.i.d. examples from $P_O$, independently of
$\calD_\ch$, and
$\varepsilon_i$ are independent uniform signs, independent of the data.
The compact continuous families above give measurable suprema, for example
by restriction to a countable dense parameter subset.

\begin{theorem}[Clean log-loss guarantee]
\label{thm:parametric-recalibration-risk}
Suppose Assumption~\ref{v03:marginal} holds, channels satisfy
$0\le\wh a_j<\wh b_j\le1$, and $\calD_\rec\indep\calD_\ch$.
Let $\calG$ be one of the families in
Eq.~\eqref{eq:temperature-family} and Eq.~\eqref{eq:platt-family}.
Suppose a measurable fitted map $\wh g\in\calG$ satisfies
$\wh R(\wh g)\le\inf_{g\in\calG}\wh R(g)+\epsilon_{\mathrm{opt}}$
with $\epsilon_{\mathrm{opt}}\ge0$.
Define
$B_{\log}\coloneqq(1+2\Bn)\log(1/\epsilon_p)$ and
$C_{\epsilon_p}\coloneqq\log((1-\epsilon_p)/\epsilon_p)$.
For $\eta\in(0,1)$, conditional on $\calD_\ch$, with probability at
least $1-\eta$,
\begin{align}
 R_{\log}(\wh g)-\inf\nolimits_{g\in\calG}R_{\log}(g)
 &\le4\Rad_{n_\rec}(\wh\calF_{\calG}\mid\calD_\ch)
 +2B_{\log}\sqrt{\frac{2\log(1/\eta)}{n_\rec}}
 +2C_{\epsilon_p}\dch+\epsilon_{\mathrm{opt}}.
 \label{eq:parametric-risk-bound}
\end{align}
For the fixed-dimensional compact families above, the complexity is
$O(n_\rec^{-1/2})$ conditional on bounded $\Bn$ and fixed parameter
and clipping bounds.
\end{theorem}
\begin{proof}
By $|\wh U|\le\Bn$,
$|\ell^{\mathrm{corr}}_{\log,\wh\rho}|\le B_{\log}$.
For a class bounded in $[-B_{\log},B_{\log}]$, symmetrization gives
expected uniform empirical deviation at most twice its expected absolute
Rademacher complexity. Changing one sample point changes the supremum
deviation by at most $2B_{\log}/n_\rec$.
McDiarmid's inequality consequently gives, with probability $1-\eta$,
$\sup_{h\in\wh\calF_{\calG}} |\E_{\calD_\rec}[h]-\E[h\mid\calD_\ch]| \le2\Rad_{n_\rec}(\wh\calF_{\calG}\mid\calD_\ch) +B_{\log}\sqrt{2\log(1/\eta)/n_\rec}.$
Also
$|\ell_1(p)-\ell_0(p)|=|\log((1-p)/p)|\le C_{\epsilon_p}$.
Theorem~\ref{thm:unbiased-corrected-loss} bounds the uniform difference
between conditional corrected risk and clean risk by
$C_{\epsilon_p}\dch$.
Call the sum of these two bounds $\Delta_n$.
For any $g\in\calG$, approximate empirical optimality gives
$R_{\log}(\wh g)\le\wh R(\wh g)+\Delta_n \le\wh R(g)+\epsilon_{\mathrm{opt}}+\Delta_n \le R_{\log}(g)+\epsilon_{\mathrm{opt}}+2\Delta_n.$
Taking the infimum proves Eq.~\eqref{eq:parametric-risk-bound}.
The bounded input logit and compact parameters make the corrected loss
uniformly Lipschitz in a one- or two-dimensional parameter.
A Euclidean cover has logarithmic size at most
$d\log(C/\epsilon')$ for dimension $d\in\{1,2\}$.
The entropy integral of its square root is finite, giving the stated
$n_\rec^{-1/2}$ order, with constants depending on the fixed bounds and
$\Bn$.
\end{proof}

\subsection{Fixed-bin recalibration}
\label{app:histogram}
Let $\calH\coloneqq\{H_h\}_{h\in[J]}$ be fixed conditional on $\calD_\ch$,
independently of $\calD_\rec$.
Define $n_h$, $\wt m_h$, and $\wh m_h$ as in Eq.~\eqref{v03:hist-est}.
For an empty bin, assign both $\wt m_h$ and $\wh m_h$ any common value
in $[0,1]$. Set
\begin{equation}\textstyle
 \wt m_\calH(z)\coloneqq\sum\nolimits_{h=1}^{J}\wt m_h\one(z\in H_h),
 \qquad
 \wh m_\calH(z)\coloneqq\sum\nolimits_{h=1}^{J}\wh m_h\one(z\in H_h).
 \label{eq:histogram-map}
\end{equation}
The first map can take values outside $[0,1]$ and is used only to describe
the empirical residual; the second is the probability-valued predictor.
For any $\calH$-binwise constant map $g$, define
\begin{equation}\textstyle
 \wh C_\calH(g)\coloneqq\sum\nolimits_{h=1}^{J}
 \left|\E_{\calD_\rec}[(\wh U-g(Z))\one(Z\in H_h)]\right|.
 \label{eq:training-bin-residual}
\end{equation}
Its bins are those of the original score used for fitting. Test ECE is
instead computed using bins of the transformed score.

\begin{theorem}[Uniform fixed-bin calibration guarantee]
\label{thm:histogram-recalibration}
Under the hypotheses of Proposition~\ref{v03:histogram}, with conditional
probability at least $1-\eta$, simultaneously for every possibly
sample-dependent $g_\calH:[0,1]\to[0,1]$ constant on each fitting bin,
\begin{equation}
 \ECE(g_\calH\circ f)
 \le\wh C_\calH(g_\calH)+\dch+
       (\Bn+1)\sqrt{\frac{2(J\log2+\log(4/\eta))}{n_\rec}}.
 \label{eq:histogram-population-bound}
\end{equation}
For the fitted maps, the empirical residuals satisfy
\begin{equation}\textstyle
 \wh C_\calH(\wt m_\calH)=0,
 \qquad
 \wh C_\calH(\wh m_\calH)
 =\sum\nolimits_{h:n_h>0}\frac{n_h}{n_\rec}|\wt m_h-\wh m_h|.
 \label{eq:histogram-projection-residual}
\end{equation}
\end{theorem}
\begin{proof}
Fix any vector of bin values $(g_h)_{h\in[J]}\in[0,1]^J$ and let
$V\coloneqq g_\calH(Z)$. The finite-valued prediction has
\[
 \ECE(g_\calH\circ f)
 =\sum\nolimits_{v\in\{g_h:h\in[J]\}}|\E[(Y-v)\one(V=v)]|
 \le\sum\nolimits_{h=1}^{J}|\E[(Y-g_h)\one(Z\in H_h)]|.
\]
The inequality allows different fitting bins to have the same predicted
value. By Lemma~\ref{lem:mean-error}, replacing $Y$ by $\wh U$ in the
population bin moments changes their sum of absolute values by at most
$\dch$. Let $p_h\coloneqq P(Z\in H_h\mid\calD_\ch)$.
Adding and subtracting the empirical moments, using $0\le g_h\le1$,
gives the pathwise bound
\begin{align*}
 \ECE(g_\calH\circ f)
 &\le\wh C_\calH(g_\calH)+\dch+D_U+D_1,\\
 D_U&\coloneqq\sum\nolimits_{h=1}^{J}|\E[\wh U\one(Z\in H_h)\mid\calD_\ch]
                      -\E_{\calD_\rec}[\wh U\one(Z\in H_h)]|,\\
 D_1&\coloneqq\sum\nolimits_{h=1}^{J}|p_h-\E_{\calD_\rec}[\one(Z\in H_h)]|.
\end{align*}
These last two quantities do not depend on the vector $(g_h)$.
Use the $2^J$ sign representation of each sum.
The corresponding summands are bounded by $\Bn$ and $1$, respectively.
Hoeffding concentration and a union bound over signs and the two sums give
$D_U+D_1\le(\Bn+1)\sqrt{2(J\log2+\log(4/\eta))/n_\rec}$
with probability at least $1-\eta$.
The same event holds for every vector of bin values, including the vector
fitted from the sample. This proves Eq.~\eqref{eq:histogram-population-bound}
for an arbitrary fitted vector of bin values.

For a nonempty bin,
$\E_{\calD_\rec}[(\wh U-g_h)\one(Z\in H_h)]
=(n_h/n_\rec)(\wt m_h-g_h)$.
Substituting $g_h=\wt m_h$ gives zero; substituting
$g_h=\Pi_{[0,1]}(\wt m_h)$ gives the second identity.
Empty bins have zero empirical contribution and are included in the
population sampling bound through $D_U,D_1$.
\end{proof}
The projection residual is determined by the fitted bin means and is zero
exactly when all nonempty-bin means lie in $[0,1]$.
No Lipschitz assumption on $m$ is used in this theorem.

\begin{lemma}[Projected bin masses]
\label{lem:projected-bin-mass}
Let $p,q\ge0$, $\theta\in[0,1]$, and $u\in\R$.
For $q>0$, set $v\coloneqq\Pi_{[0,1]}(u/q)$.
For $q=0$, assume $u=0$ and choose any $v\in[0,1]$.
Then
\begin{equation}\textstyle
 p|v-\theta|\le|u-p\theta|+|q-p|.
 \label{eq:projected-bin-mass}
\end{equation}
\end{lemma}
\begin{proof}
If $q=0$, the left side is at most $p$, which is no larger than the
right side. For $q>0$ and $0\le u\le q$, $u=qv$ and
$p(v-\theta)=u-p\theta+(p-q)v$, proving the inequality.
If $u<0$, then $v=0$ and $p\theta\le|u-p\theta|$.
If $u>q$, then $v=1$ and
$p(1-\theta)=(p-q)+(q-u)+(u-p\theta)
\le|p-q|+|u-p\theta|$.
\end{proof}

\begin{corollary}[Direct guarantee for the projected histogram]
\label{cor:projected-histogram}
Under the hypotheses of Proposition~\ref{v03:histogram}, with conditional
probability at least $1-\eta$,
\begin{equation}
 \ECE(\wh m_\calH\circ f)
 \le\dch+(\Bn+1)
          \sqrt{\frac{2(J\log2+\log(4/\eta))}{n_\rec}}.
 \label{eq:projected-histogram-bound}
\end{equation}
No empirical projection residual appears in this bound.
\end{corollary}
\begin{proof}
For each fitting bin, let
$p_h\coloneqq P(Z\in H_h\mid\calD_\ch)$,
$\theta_h\coloneqq\E[Y\mid Z\in H_h,\calD_\ch]$ when $p_h>0$, and
choose any $\theta_h\in[0,1]$ when $p_h=0$.
Apply Lemma~\ref{lem:projected-bin-mass} with
$q_h\coloneqq n_h/n_\rec$,
$u_h\coloneqq\E_{\calD_\rec}[\wh U\one(Z\in H_h)]$, and
$v_h\coloneqq\wh m_h$. If $q_h=0$ then $u_h=0$, so the lemma includes empty bins.
Grouping bins with equal outputs and using the triangle inequality gives
\begin{align*}
 \ECE(\wh m_\calH\circ f)
 &\le\sum\nolimits_{h=1}^{J}p_h|\wh m_h-\theta_h|\\
 &\le\sum\nolimits_{h=1}^{J}|u_h-p_h\theta_h|
         +\sum\nolimits_{h=1}^{J}|q_h-p_h|\le D_U+D_1+\dch.
\end{align*}
The last step uses Eq.~\eqref{eq:conditional-channel-bias} to compare
$\E[\wh U\one(Z\in H_h)\mid\calD_\ch]$ with
$p_h\theta_h$. The $D_U,D_1$ terms and their simultaneous concentration
are exactly those in the proof of
Theorem~\ref{thm:histogram-recalibration}.
This proves Eq.~\eqref{eq:projected-histogram-bound} and
Proposition~\ref{v03:histogram}.
\end{proof}
The general residual bound applies to any binwise-constant map.
For the specific projected averages, the preceding proof removes the
empirical projection-residual term. Projection is already included in the
comparison with the bin's true positive-label frequency.
The absence of a bin-approximation term concerns calibration of a
coarsened output. It is not a claim that coarse predictions retain the
information or discrimination of the original predictor.

\begin{corollary}[Rate for projected histogram recalibration]
\label{cor:histogram-rate}
Under Proposition~\ref{v03:histogram}, suppose $J,K$ are fixed,
$\min_j\delta_j\ge\delta_{\min}>0$, and
$e_\ch=O_p(\varepsilon_{n_\ch})$, with $e_\ch$ as in Eq.~\eqref{eq:channel-estimation-error},
for a positive deterministic sequence $\varepsilon_{n_\ch}\to0$. Then
$\ECE(\wh m_\calH\circ f) =O_p(n_\rec^{-1/2}+\varepsilon_{n_\ch}).$
No smoothness of $m$ is required.
\end{corollary}
\begin{proof}
The channel-error assumption gives $\Bn=O_p(1)$ and
$\dch=O_p(\varepsilon_{n_\ch})$ by the proof of
Corollary~\ref{cor:ece-rate-main}. Substitute these orders and fixed $J$
into Proposition~\ref{v03:histogram}.
\end{proof}

\subsection{Evaluation of a fitted predictor}
\label{app:fitted-evaluation}
Let $\mathcal A$ be the sigma-field generated by all data and algorithmic
randomness used to choose a fitted map $g$, channel estimates, convex
weights, and evaluation partition $\calB$. The map is jointly measurable
in this information and its score argument. Assume the i.i.d. evaluation examples are jointly
independent of $\mathcal A$. For an independent example let $T\coloneqq g(Z)$.
Every expectation below holds $\mathcal A$ fixed, and
$\ECE_\calB(g\circ f)=\sum_{\ell=1}^{M} |\E[(Y-T)\one(T\in B_\ell)\mid\mathcal A]|.$

\begin{proposition}[Held-out evaluation after recalibration]
\label{prop:fitted-evaluation}
Suppose Assumption~\ref{v03:marginal} holds and the estimated channels
satisfy $0\le\wh a_j<\wh b_j\le1$ for $j\in[K]$.
On the independent sample, define
$\wh\ECE^{\nc}_\calB(g\circ f,\calD_\ev)\coloneqq\sum_{\ell=1}^{M} \left|n_\ev^{-1}\sum_{i=1}^{n_\ev} (\wh U_i-g(Z_i))\one(g(Z_i)\in B_\ell)\right|.$
For every $\eta\in(0,1)$, conditional on $\mathcal A$, with probability
at least $1-\eta$,
\begin{equation}
 |\wh\ECE^{\nc}_\calB(g\circ f,\calD_\ev)-\ECE_\calB(g\circ f)|
 \le\Bn\sqrt{2(M\log2+\log(2/\eta))/n_\ev}+\dch.
 \label{eq:fitted-evaluation-bound}
\end{equation}
If $g$ has at most $J$ distinct output values and every evaluation bin
contains at most one such value, then
$\ECE_\calB(g\circ f)=\ECE(g\circ f)$.
For a general map, an $L_g$-Lipschitz version of
$t\mapsto\E[Y\mid T=t,\mathcal A]-t$ instead gives the approximation
term $L_gh_\calB$.
\end{proposition}
\begin{proof}
For each fitted correction, Eq.~\eqref{eq:conditional-channel-bias} gives
$|\E[\wh U-Y\mid Z,\mathcal A]|\le\dch$.
Since $T$ is a function of $Z$ conditional on $\mathcal A$,
iterated expectation yields
$|\E[\wh U-Y\mid T,\mathcal A]|\le\dch$.
Also $|\wh U-T|\le\Bn$ by the same convex-endpoint calculation as in
Lemma~\ref{lem:mean-error}. The bin-sign proof of
Theorem~\ref{thm:arbitrary-rho-bound} now applies to $T$.
For finitely many outputs, a partition separating them groups no distinct
prediction values, so the bin-moment expression is exactly the conditional
ECE. Such interval boundaries can be placed between consecutive distinct
outputs using fitting data alone; repeated values are merged.
For the Lipschitz case, apply Eq.~\eqref{eq:bin-approximation} to the transformed residual.
\end{proof}
Fixed equal-width test bins need not separate the outputs of a histogram
map. The exact finite-output statement applies only to a partition that
does so. Even when there is no binning error, empirical ECE still has sampling
error. The finite-output statement does not identify a finite-sample ECE
value with exact population ECE.

\subsection{Split-sample workflow}
\label{app:workflow}
The fitting and evaluation partitions have different roles. In
Algorithm~\ref{v03:pipeline}, all choices made before evaluation are
independent of $\calD_\ev$.
A separate channel sample may be used for final evaluation, as in the
human-annotation protocol. In that case the held-out channels replace the
fitting channels only in the evaluation step.
Proposition~\ref{prop:fitted-evaluation} gives the evaluation guarantee conditional on the samples determining the fitted map, weights, and test partition.

\begin{algorithm}[H]
\caption{Calibration from repeated noisy annotations}
\label{v03:pipeline}
\begin{algorithmic}[1]
\Require Fixed independently trained predictor $f$; independent samples
$\calD_s$ of sizes $n_s$, $s\in\{\ch,\rec,\val,\ev\}$;
identification groups independent of $\calD_\ch$;
fitting bins $\calH$ independent of $\calD_\rec$;
evaluation bins $\calB$ independent of $\calD_\ev$;
a pre-specified family of candidate procedures.
\State Estimate channels $(\wh a_j,\wh b_j)_{j\in[K]}$ and weights
$\wh\alpha\in\Delta_K$ from $\calD_\ch$.
\State Verify the required estimated gaps; return failure if the declared
admissibility conditions are not met.
\For{$s\in\{\rec,\val,\ev\}$ and $i\in[n_s]$}
\State Compute $\wh U_i^{(s)}\gets \sum_{j=1}^{K}\wh\alpha_j
(A_{ij}^{(s)}-\wh a_j)/(\wh b_j-\wh a_j)$.
\EndFor
\State Fit candidate maps on $\calD_\rec$ by corrected loss or the projected
bin averages in Eq.~\eqref{eq:histogram-map}.
\State Select $\wh g$ using the pre-specified corrected validation criterion
on $\calD_\val$, or retain the single pre-specified candidate.
\For{$i\in[n_\ev]$}
\State Set $T_i\gets \wh g(Z_i^{(\ev)})$.
\EndFor
\For{$\ell\in[M]$}
\State Compute $\wh r_\ell\gets n_\ev^{-1}\sum_{i=1}^{n_\ev}
(\wh U_i^{(\ev)}-T_i)\one(T_i\in B_\ell)$.
\EndFor
\State Set $\wh\ECE^{\nc}\gets \sum_{\ell=1}^M|\wh r_\ell|$.
\State \Return $(\wh g,\wh\ECE^{\nc})$ and channel-fitting diagnostics.
\end{algorithmic}
\end{algorithm}

\section{Algorithms and Implementation}
\label{v03:algorithms}\label{app:em}
\subsection{Likelihood updates}
For $n=n_\ch$, let $A_i=(A_{i1},\ldots,A_{iK})$ and define
\begin{equation}\textstyle
 L_{i0}(\mathbf a)=\prod\nolimits_j a_j^{A_{ij}}(1-a_j)^{1-A_{ij}},\qquad
 L_{i1}(\mathbf b)=\prod\nolimits_j b_j^{A_{ij}}(1-b_j)^{1-A_{ij}}.
 \label{eq:em-components}
\end{equation}
The E-step evaluates a posterior class probability under the current parameters;
the M-step maximizes the expected complete-data log likelihood
\citep{dempster1977maximum}. These posterior probabilities lie in $[0,1]$;
the corrected labels need not.

\subsubsection{Marginal Dawid--Skene}
\label{app:em-marginal}
The response likelihood is $(1-\mu)L_{i0}+\mu L_{i1}$. At iteration $q$,
\begin{equation}
 \gamma_i^{(q)}=\frac{\mu^{(q)}L_{i1}(\mathbf b^{(q)})}
 {(1-\mu^{(q)})L_{i0}(\mathbf a^{(q)})+\mu^{(q)}L_{i1}(\mathbf b^{(q)})}.
 \label{eq:marginal-responsibility}
\end{equation}
Replacing the unobserved $Y_i$ by $\gamma_i^{(q)}$ in the Bernoulli log likelihood gives
\begin{equation}
 \mu^{(q+1)}=n^{-1}\sum\nolimits_i\gamma_i^{(q)},\quad
 a_j^{(q+1)}=\frac{\sum\nolimits_i(1-\gamma_i^{(q)})A_{ij}}{\sum\nolimits_i(1-\gamma_i^{(q)})},\quad
 b_j^{(q+1)}=\frac{\sum\nolimits_i\gamma_i^{(q)}A_{ij}}{\sum\nolimits_i\gamma_i^{(q)}}.
 \label{eq:shared-channel-updates}
\end{equation}
No predictions enter these updates. A marginal two-annotator fit is used only
as a non-identified control.

\subsubsection{Shared channels with score groups}
\label{app:em-grouped}
For fixed groups $C_g$, write $G_i=g$ when $Z_i\in C_g$ and $n_g=|\{i:G_i=g\}|>0$.
Use the grouped likelihood in Eq.~\eqref{eq:grouped-likelihood}. Its updates are
\begin{equation}
 \gamma_i^{(q)}=\frac{\mu_{G_i}^{(q)}L_{i1}(\mathbf b^{(q)})}
 {(1-\mu_{G_i}^{(q)})L_{i0}(\mathbf a^{(q)})+\mu_{G_i}^{(q)}L_{i1}(\mathbf b^{(q)})},
 \quad
 \mu_g^{(q+1)}=n_g^{-1}\sum\nolimits_{i:G_i=g}\gamma_i^{(q)}.
 \label{eq:grouped-responsibility}
\end{equation}
The channel updates remain Eq.~\eqref{eq:shared-channel-updates}, pooled over
all examples. Joint score-bin EM uses this same likelihood. Its optional
prediction map is
\begin{equation}\textstyle
 \wh m_\calC(z)=\sum\nolimits_{g=1}^L\wh\mu_g\one(z\in C_g).
 \label{eq:joint-map}
\end{equation}
This map is distinct from separately fitting a recalibration map to corrected
labels. Its in-sample summary $\sum_g(n_g/n)|\wh\mu_g-\bar Z_g|$ is not an
independent evaluation. Fitting group-specific channels instead applies the
marginal updates within each group, with a common label orientation; the
shared-channel guarantees do not automatically cover that extension.

\subsection{Numerical procedure}
\label{app:numerical-em}
For each EM start, alternate the updates above using log-scale probabilities
and log-sum-exp until the stopping tolerance or iteration limit is reached.
Impose the prescribed orientation by a single global label exchange,
$\mu_g\leftarrow1-\mu_g$ and $(a_j,b_j)\leftarrow(b_j,a_j)$; the marginal
model has one class proportion. Candidate selection and acceptance criteria
depend on the estimator and experiment. Tabular grouped EM selects the
largest-likelihood candidate after orientation and preliminary gap screening,
then checks convergence and the final gap threshold for that candidate.
Marginal Dawid--Skene uses finite-parameter and gap checks without a separate
convergence check. Further experiment-specific details appear in
Appendix~\ref{v03:experimental-details}. Two-annotator GMM minimizes the
constrained objective in Definition~\ref{def:gmm-main}; its optimization-error
term is given in Theorem~\ref{thm:gmm-finite}.

\subsection{Computational cost}
\label{app:em-cost}
For $R$ starts and $Q$ iterations, marginal and shared-channel EM require
$O(RQnK)$ operations and $O(nK+n+K+L)$ storage. Group-specific fits have the
same time order and require $O(nK+n+KL)$ storage. Two-annotator moment
aggregation costs $O(n+L)$ before optimization. Given bin memberships,
\NCECE{} costs $O(nK+M)$; assigning equal-width bins costs $O(n)$.

\subsection{Feasible aggregation weights}
\label{app:feasible-weights}
For an estimated class proportion $\wh\mu$, define
\begin{equation}
 \wh v_j=\frac{(1-\wh\mu)\wh a_j(1-\wh a_j)+\wh\mu\wh b_j(1-\wh b_j)}
 {\wh\delta_j^2},\qquad
 \wh\alpha_j^{\IV}=\frac{\wh v_j^{-1}}{\sum\nolimits_k\wh v_k^{-1}}.
 \label{eq:feasible-iv}
\end{equation}
For zero variances, use the support convention of
Lemma~\ref{lem:variance-optimal-alpha} or the stabilization below.
Gaussian recalibration uses
$\wh\mu=\sum_g(n_g/n_\ch)\wh\mu_g$ and
\begin{equation}\textstyle
 \wh\alpha^{\mathrm{shr}}(\lambda)=(1-\lambda)\wh\alpha^{\IV}
       +\lambda(K^{-1},\ldots,K^{-1}),\qquad \lambda=0.25.
 \label{eq:shrunk-weights}
\end{equation}
The tabular Main rule estimates $\wh\mu$ by the mean of equally weighted
corrected labels, clipped to $[10^{-4},1-10^{-4}]$, and floors $\wh v_j$ at
$10^{-8}$ before inversion. Equal, Gap, and Squared-gap weights are
proportional to $1$, $\wh\delta_j$, and $\wh\delta_j^2$, respectively.
Weights are fixed before test evaluation.

\section{Experimental Protocols and Additional Results}
\label{v03:experimental-details}\label{v03:supp-experiments}

\subsection{Evaluation and reporting conventions}
\label{app:experimental-targets}
Toy estimation errors use population $\ECE(f)=\E|m(Z)-Z|$ as the reference.
Gaussian Oracle ECE replaces test labels by $m(Z_i)$ in empirical bin moments.
Tabular and CIFAR Oracle ECE use original test labels. These empirical
references use the same predictions and bins as the noisy-label criteria.
``Oracle channels'' means that the true noise probabilities are supplied.

For repetition errors $e_r$, bias and MAE are the means of $e_r$ and
$|e_r|$, respectively. RMSE is the square root of the mean of $e_r^2$.
Figure error bars show one Monte Carlo standard error of MAE.
Table~\ref{v31:tab-reuse} reports sample SDs over 200 repetitions;
Table~\ref{tab:recalibration-full} reports sample SDs over 60 repetitions.
Tabular and CIFAR tables use the aggregation units described below and in
their captions.

\subsection{Synthetic experiments}
\label{v03:experiment-protocol}\label{app:experiment-details}
\label{v03:channel-experiments}

\paragraph{Gaussian recalibration.}
\label{app:toy-recalibration}
We use $P(Y=1)=1/2$, $X\mid Y=1\sim N(-1,1)$,
$X\mid Y=0\sim N(1,1)$, and $Z=\sigma(0.3-X)$
\citep{futami2024information}, giving
$m(z)=\sigma(2\logit(z)-0.6)$. The two annotators are independent given $Y$ and have
$(\rho_{0,1},\rho_{1,1})=(0.05,0.25)$ and
$(\rho_{0,2},\rho_{1,2})=(0.05,0.05)$.
For total sizes $1500,3000,6000,12000$, each of 60 repetitions uses equal
independent channel-fitting, recalibration, and test thirds.
Two-group constrained GMM uses the fixed population score-median boundary
$\sigma(0.3)$ and four starts. Corrected fitting uses the shrunk IV weights
above; Naive fitting uses average annotations, and clean-label fits use $Y$.

Temperature fitting uses bounded scalar optimization with inverse temperature
in $[0.05,20]$. Platt fitting uses L-BFGS-B with slope in $[0,10]$,
intercept in $[-10,10]$, and four starts.
Projected histogram fitting uses $J=\operatorname{round}(n_\rec^{1/3})$
population score-quantile bins. Evaluation uses
$M=\operatorname{round}(n_\ev^{1/3})$ equal-width bins of each transformed
score; at total size $12000$, $J=M=16$.
The $L_1(m)$ metric is the test mean of $|m(Z_i)-g(Z_i)|$.
Log loss and Brier use sampled test labels; log-loss evaluation clips
predictions to $[10^{-8},1-10^{-8}]$.
Table~\ref{tab:recalibration-full} reports the largest total size.
\begin{table}[htbp]
\centering\small
\caption{Gaussian recalibration at $n=12000$: mean $\pm$ SD over 60 repetitions (unscaled metrics).}
\label{tab:recalibration-full}
\label{tab:recalibration-main}
\setlength{\tabcolsep}{3pt}
\begin{tabular}{@{}lrrr@{}}
\toprule
\multicolumn{4}{c}{Calibration errors}\\
\midrule
Method & Oracle ECE & \NCECE{} & Abs. error\\
\midrule
Clean-label Platt (oracle fit) & $0.006118 \pm 0.003526$ & $0.022761 \pm 0.004720$ & $0.016642 \pm 0.005200$\\
Corrected Platt & $\boldsymbol{0.009687 \pm 0.005702}$ & $0.022462 \pm 0.004813$ & $0.012868 \pm 0.005686$\\
Corrected histogram $\widehat m$ & $0.016301 \pm 0.003894$ & $\boldsymbol{0.022147 \pm 0.005481}$ & $0.006692 \pm 0.004531$\\
Clean-label temperature (oracle fit) & $0.065946 \pm 0.001003$ & $0.066361 \pm 0.009310$ & $0.007205 \pm 0.005366$\\
Corrected temperature & $0.066095 \pm 0.001115$ & $0.066402 \pm 0.009139$ & $0.007033 \pm 0.005327$\\
Naive Platt & $0.092819 \pm 0.004540$ & $0.094160 \pm 0.007815$ & $\boldsymbol{0.004922 \pm 0.003826}$\\
Before recalibration & $0.113960 \pm 0.000804$ & $0.115110 \pm 0.006985$ & $0.005518 \pm 0.004240$\\
Naive temperature & $0.114570 \pm 0.004091$ & $0.115727 \pm 0.007980$ & $0.005492 \pm 0.004131$\\
\bottomrule
\end{tabular}
\par\smallskip
\begin{tabular}{@{}lrrr@{}}
\toprule
\multicolumn{4}{c}{Curve error and predictive losses}\\
\midrule
Method & $L_1(m)$ & Log loss & Brier\\
\midrule
Clean-label Platt (oracle fit) & $0.006125 \pm 0.003528$ & $0.356120 \pm 0.007968$ & $0.112331 \pm 0.002768$\\
Corrected Platt & $\boldsymbol{0.009696 \pm 0.005709}$ & $\boldsymbol{0.356601 \pm 0.008111}$ & $\boldsymbol{0.112481 \pm 0.002789}$\\
Corrected histogram $\widehat m$ & $0.025171 \pm 0.003186$ & $0.369160 \pm 0.014568$ & $0.113294 \pm 0.002773$\\
Clean-label temperature (oracle fit) & $0.065958 \pm 0.000988$ & $0.375123 \pm 0.008151$ & $0.119115 \pm 0.002830$\\
Corrected temperature & $0.066124 \pm 0.001079$ & $0.375399 \pm 0.008352$ & $0.119182 \pm 0.002870$\\
Naive Platt & $0.093059 \pm 0.004546$ & $0.397467 \pm 0.006915$ & $0.123929 \pm 0.002757$\\
Before recalibration & $0.114404 \pm 0.000803$ & $0.414317 \pm 0.005252$ & $0.129034 \pm 0.002156$\\
Naive temperature & $0.114997 \pm 0.003862$ & $0.414834 \pm 0.006053$ & $0.129197 \pm 0.002305$\\
\bottomrule
\end{tabular}
\end{table}

\paragraph{Toy ECE estimation and sample reuse.}
\label{v31:controlled}\label{v31:implementation}
The input distribution, predictor, three label functions, and annotator
channels are those of Section~\ref{v03:experiments}.
Five equal-width channel groups and ten equal-width evaluation bins are fixed
before sampling. Each $N\in\{500,1000,2000,4000,8000\}$ has 200 repetitions.
The population ECEs for Calibrated, Shifted, and
Nonlinear are $0$, $0.08962637465$, and $0.08328889713$.
For splitting, the first half fits channels and the second half evaluates
all criteria. Reuse applies both steps to the full sample. Five-fold
cross-fitting pools held-out signed bin moments with fold-size weights
before taking absolute values. \NCECE{} uses equal annotator weights.

Joint EM fits shared channels and group-specific class proportions from
response-pattern counts. Each fit uses ten starts, at most $1000$ iterations,
and stopping threshold $10^{-9}(1+|\ell_{\mathrm{previous}}|)$ for the change
in mean log likelihood. Probabilities are clipped to $[10^{-6},1-10^{-6}]$.
Initial $a_j,b_j$ are uniform on $[0.04,0.30]$ and $[0.70,0.96]$;
initial proportions are uniform on $[0.1,0.9]$, with group midpoints used
for the first start. The first channel fixes global orientation.
Retain the largest likelihood among converged starts with all oriented gaps
at least $10^{-3}$. Empty groups or no accepted start trigger the fixed
fallback $(\wh a_j,\wh b_j)=(0.2,0.8)$ and proportions $1/2$.
All 200 fits succeeded in each reported condition.
Corrected labels and bin means are unprojected.
Figure~\ref{v31:fig-other-functions} reports Shifted and Nonlinear.

\begin{figure}[htbp]
\centering
\includegraphics[width=0.48\textwidth]{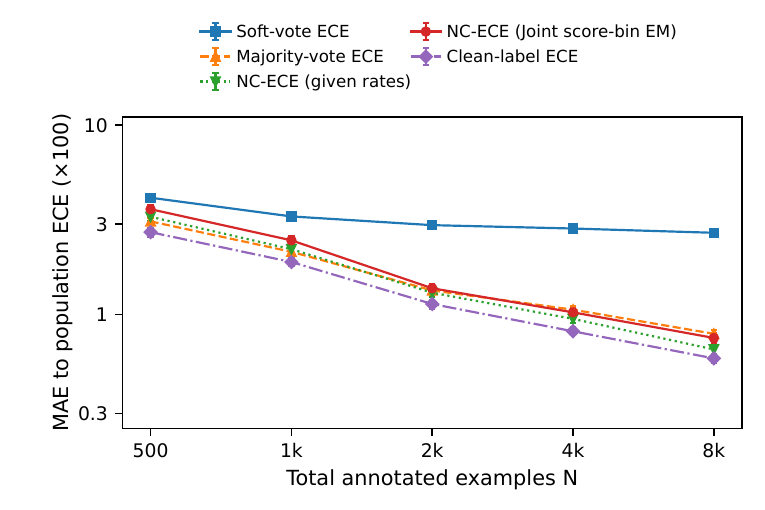}\hfill
\includegraphics[width=0.48\textwidth]{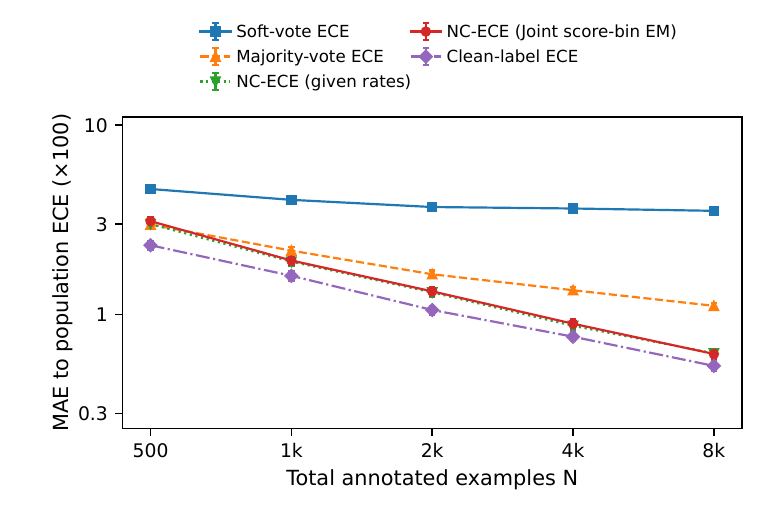}
\caption{ECE-estimation MAE ($\times100$): Shifted (left) and Nonlinear (right).}
\label{v31:fig-other-functions}
\end{figure}

\subsection{Semi-synthetic tabular benchmarks}
\label{v03:realdata}
ACSIncome, CreditDefault, HMDA, and MEPS use the benchmark preprocessing and
six predictor families
\citep{ding2021retiring,yeh2009default,hansen2024multicalibration,sharma2021fairn}.
Base models use original labels; the MLP checkpoint is selected by validation
accuracy, and the SVM produces clipped hard predictions.
Post-hoc channel estimation and recalibration use simulated annotations.
We use ten seeds, $K=2,3,4$, and four channel regimes, taking the first $K$
pairs $(\rho_0,\rho_1)$ below:
\begin{center}
\scriptsize
\begin{tabular}{@{}ll@{}}
Symmetric & $(.05,.05),(.10,.10),(.15,.15),(.20,.20)$\\
Asymmetric & $(.05,.25),(.25,.05),(.10,.20),(.20,.10)$\\
Heterogeneous & $(.02,.08),(.08,.20),(.20,.08),(.28,.22)$\\
Hard & $(.20,.30),(.30,.20),(.25,.30),(.30,.25)$
\end{tabular}
\end{center}

Base-training, channel-fitting, recalibration, validation, and test fractions
are $0.36/0.12/0.12/0.20/0.20$. Test uses seed $42$, validation the run seed,
the remaining training division seed $50$, and the channel/recalibration
division seed $51$ plus the run seed. Each post-hoc split receives independently
generated annotations. Predictors share annotation matrices within each
dataset, seed, regime, and split; $K=2,3$ uses the first columns of $K=4$.

The primary analysis uses Joint score-bin EM and Main plug-in IV weights.
Joint EM estimates shared channels and prevalences for the nonempty bins of
ten equal-width score bins on the recalibration split. Temperature, Platt,
and projected histogram maps are fitted on this split; the auxiliary Joint
map uses the fitted prevalences, with empty-bin values set to the overall
fitted posterior mean. Naive fitting uses mean
annotator responses, and Corrected fitting uses IV-weighted corrected labels.
Oracle-channel comparisons cover all $K$, constrained GMM covers $K=2$,
and marginal Dawid--Skene EM covers $K=3,4$.
Marginal DS and GMM use the channel-fitting split. GMM forms two equal-count
groups from score ranks with deterministic tie breaking. The tabular GMM fit
also constrains the fitted prevalence in the higher-score group to exceed
that in the lower-score group by at least $10^{-5}$.
Sample reuse and adaptive GMM groups differ from the independent-sample and
fixed-group conditions of the corresponding theoretical guarantees.
Joint EM, histogram fitting, and test ECE use ten equal-width bins;
test bins are applied to each transformed score.

Channel fits require finite parameters and estimated gaps greater than
$10^{-4}$, with estimator-specific success checks. Inadmissible channels stop
that grid cell. Temperature and Platt fits retain finite objectives and
parameters. Each dataset covers
$720=6\times10\times3\times4$ predictor--seed--annotator-count--regime cells,
and pooled summaries cover $2880$ cells.
Recalibration means and sample SDs use ten seed-level means, averaging the
predictor, annotator-count, and regime grid within each seed; pooled results
also average across datasets within each seed. Channel-RMSE summaries use the
grid units specified in their captions. Annotations and test data are shared
across conditions.
\begin{table}[htbp]
\centering\small
\caption{Tabular recalibration: mean $\pm$ SD over ten seeds. ECE and absolute gaps are multiplied by 100; accuracy is in percent.}
\label{tab:realdata-recalibration}
\begin{tabular}{lrrrr}
\toprule
Method & Oracle ECE & \NCECE{} & Absolute gap & Accuracy (\%) \\
\midrule
Base & $8.03 \pm 0.27$ & $8.25 \pm 0.27$ & $0.50 \pm 0.08$ & $81.75 \pm 0.25$ \\
Naive temperature & $14.79 \pm 0.17$ & $14.77 \pm 0.26$ & $0.62 \pm 0.10$ & $81.75 \pm 0.25$ \\
Corrected temperature & $6.27 \pm 0.29$ & $6.32 \pm 0.25$ & $0.55 \pm 0.07$ & $81.75 \pm 0.25$ \\
Naive Platt & $12.04 \pm 0.08$ & $12.05 \pm 0.22$ & $0.54 \pm 0.08$ & $82.15 \pm 0.05$ \\
Corrected Platt & $2.51 \pm 0.13$ & $2.48 \pm 0.09$ & $\mathbf{0.42 \pm 0.08}$ & $82.40 \pm 0.06$ \\
Naive fixed-bin & $10.05 \pm 0.07$ & $10.07 \pm 0.27$ & $0.57 \pm 0.09$ & $82.47 \pm 0.08$ \\
Corrected fixed-bin & $1.37 \pm 0.13$ & $1.33 \pm 0.08$ & $0.43 \pm 0.08$ & $\mathbf{82.48 \pm 0.09}$ \\
Joint map & $\mathbf{1.35 \pm 0.13}$ & $1.32 \pm 0.07$ & $0.43 \pm 0.08$ & $\mathbf{82.48 \pm 0.08}$ \\
\bottomrule
\end{tabular}
\end{table}

\begin{table}[htbp]
\centering
\small
\caption{Tabular channel estimation and ECE evaluation. Channel RMSE is mean $\pm$ SD over 960 dataset--predictor--seed--noise settings for each $K$; \NCECE{}--Oracle absolute gaps are means. All metrics are unscaled.}
\label{tab:realdata-channel-direct-app}
\begin{tabular}{lccc}
\toprule
Channel method & $K$ & Channel RMSE & Mean absolute gap\\
\midrule
Oracle channels & 2 & $0.0000 \pm 0.0000$ & $0.0031$\\
Oracle channels & 3 & $0.0000 \pm 0.0000$ & $0.0026$\\
Oracle channels & 4 & $0.0000 \pm 0.0000$ & $0.0023$\\
Constrained GMM & 2 & $0.0264 \pm 0.0333$ & $0.0101$\\
Marginal Dawid--Skene EM & 3 & $0.0165 \pm 0.0195$ & $0.0064$\\
Marginal Dawid--Skene EM & 4 & $0.0124 \pm 0.0122$ & $0.0046$\\
Joint score-bin EM & 2 & $\boldsymbol{0.0165 \pm 0.0182}$ & $\boldsymbol{0.0062}$\\
Joint score-bin EM & 3 & $\boldsymbol{0.0123 \pm 0.0121}$ & $\boldsymbol{0.0048}$\\
Joint score-bin EM & 4 & $\boldsymbol{0.0118 \pm 0.0102}$ & $\boldsymbol{0.0039}$\\
\bottomrule
\end{tabular}
\end{table}

\subsection{CIFAR-10N stress test}
\label{v03:human-protocol}\label{v03:human}
CIFAR-10N supplies $50{,}000$ images and three response columns
\citep{wei2022learning}. Labels are mapped to animal (bird, cat, deer, dog,
frog, horse) versus vehicle (airplane, automobile, ship, truck).
CLIP ViT-B/32 \citep{radford2021learning} predicts animal probability by
summing the six animal probabilities after softmax over ten prompts
\texttt{a photo of a \{class\}}, without fine-tuning.
Original CIFAR-10 labels \citep{krizhevsky2009learning} provide the evaluation
reference. Response columns are treated as slots because workers can vary
across images or batches.

Disjoint pilot, evaluation-channel, and test partitions contain
$10{,}000$, $20{,}000$, and $20{,}000$ images, retaining the original ten-image
annotation batches. Five pilot splits provide channel-fitting, recalibration,
and validation subsets. Marginal Dawid--Skene EM estimates channels from
responses alone; Joint score-bin EM is a separate comparison. 
Corrected labels are aggregated using IV weights.
Joint EM uses five equal-width channel groups, ten starts, at most
$1000$ iterations, tolerance $10^{-8}$, and converged fits with positive gaps.

Within each family, Naive and Corrected methods share fitting examples and
use average responses and IV-weighted corrected labels, respectively.
Temperature and Platt minimize logarithmic loss with $0.05\le T\le20$ or
$-10\le u,v\le10$. Fixed-bin fitting uses ten equal-width original-score
bins and projected mean targets; empty bins receive the projected mean of
all fitting targets. The auxiliary Joint map uses the five fitted prevalences
on the pilot channel sample.

Validation \NCECE{} compares Base and the six Naive/Corrected candidates.
Table~\ref{v03:tab-human} shows the fixed-bin comparison, and
Table~\ref{tab:uncertainty-cifar} shows all candidates.
The evaluation-channel sample supplies a common correction after fitting.
For each map $g$, ten equal-width bins of $g(Z)$ are used on the test sample;
Oracle, Soft-vote, and \NCECE{} share predictions and bins, changing only
the label targets. The five pilot fits share evaluation-channel and test
samples; reported sample SDs measure variation across pilot fits.
\begin{table}[htbp]
\centering\small
\caption{CIFAR-10N recalibration with marginal Dawid--Skene EM and Joint score-bin EM: mean $\pm$ SD over five pilot splits (unscaled).}
\label{tab:uncertainty-cifar}
\setlength{\tabcolsep}{3pt}
\begin{tabular}{@{}lrrr@{}}
\toprule
\multicolumn{4}{c}{Marginal Dawid--Skene EM}\\
\midrule
Method & Oracle ECE & Soft-vote ECE & \NCECE{}\\
\midrule
Base & $0.02129 \pm 0.00000$ & $0.02825 \pm 0.00000$ & $0.02436 \pm 0.00000$\\
Naive fixed-bin & $0.04351 \pm 0.00243$ & $\boldsymbol{0.00450 \pm 0.00141}$ & $0.04695 \pm 0.00226$\\
Naive temperature & $0.05773 \pm 0.00220$ & $0.02329 \pm 0.00082$ & $0.06076 \pm 0.00220$\\
Naive Platt & $0.05754 \pm 0.00235$ & $0.02137 \pm 0.00057$ & $0.06070 \pm 0.00234$\\
Corrected temperature & $0.00557 \pm 0.00109$ & $0.04709 \pm 0.00110$ & $0.00415 \pm 0.00024$\\
Corrected Platt & $0.00644 \pm 0.00033$ & $0.04782 \pm 0.00037$ & $\boldsymbol{0.00357 \pm 0.00079}$\\
Corrected fixed-bin & $\boldsymbol{0.00470 \pm 0.00157}$ & $0.04291 \pm 0.00101$ & $0.00482 \pm 0.00069$\\
Joint EM map (auxiliary) & -- & -- & --\\
\bottomrule
\end{tabular}
\par\smallskip
\begin{tabular}{@{}lrrr@{}}
\toprule
\multicolumn{4}{c}{Joint score-bin EM}\\
\midrule
Method & Oracle ECE & Soft-vote ECE & \NCECE{}\\
\midrule
Base & $0.02129 \pm 0.00000$ & $0.02825 \pm 0.00000$ & $0.02433 \pm 0.00000$\\
Naive fixed-bin & $0.04351 \pm 0.00243$ & $\boldsymbol{0.00450 \pm 0.00141}$ & $0.04693 \pm 0.00225$\\
Naive temperature & $0.05773 \pm 0.00220$ & $0.02329 \pm 0.00082$ & $0.06074 \pm 0.00220$\\
Naive Platt & $0.05754 \pm 0.00235$ & $0.02137 \pm 0.00057$ & $0.06067 \pm 0.00234$\\
Corrected temperature & $0.00596 \pm 0.00047$ & $0.04748 \pm 0.00047$ & $\boldsymbol{0.00340 \pm 0.00015}$\\
Corrected Platt & $0.00645 \pm 0.00035$ & $0.04790 \pm 0.00043$ & $0.00349 \pm 0.00032$\\
Corrected fixed-bin & $\boldsymbol{0.00419 \pm 0.00095}$ & $0.04313 \pm 0.00099$ & $0.00465 \pm 0.00088$\\
Joint EM map (auxiliary) & $0.00412 \pm 0.00160$ & $0.04270 \pm 0.00108$ & $0.00527 \pm 0.00199$\\
\bottomrule
\end{tabular}
\end{table}

\section{Additional Related Work and Discussion}
\label{v03:background}
\subsection{Related work}
\label{v03:related}\label{v03:source-01}\label{sec:related}
Binning, smoothing, and conditional-probability estimation affect calibration
metrics and their estimators
\citep{Vaicenavicius2019,roelofs2022mitigating,popordanoska2022consistent,
popordanoska2024proper}. Weighted-moment formulations connect ECE with other
calibration criteria
\citep{blasiok2023unifying,gopalan2022lowdegree,rossellini2025testable,hegazy2025utility}.
Here label noise changes the moment's label-side expectation; correcting it
leaves binning and sampling errors as distinct terms.

Class-conditional loss correction uses known transition probabilities to recover
clean risk \citep{natarajan2013learning}. Calibration studies also address
noisy-label training, noisy validation labels, and noisy predictions
\citep{olmin2022robustness,bae2022noisy,penso2024confidence}.
Positive--unlabeled ECE estimation \citep{kiryo2025pu} and weakly supervised
multicalibration \citep{futami2026weak} use other supervision structures.
Our repeated annotations support estimation of unknown annotator-specific
channels and propagation of that estimation error to ECE and recalibration.

The identification routes specialize the Hui--Walter and Dawid--Skene models
\citep{hui1980estimating,dawid1979maximum,allman2009identifiability}; the GMM
analysis quantifies their stability \citep{hansen1982gmm}.
Unlike joint predictor-and-annotator training \citep{tanno2019learning}, the
predictor here is fixed independently of channel estimation and recalibration.
Disagreement-based methods can instead treat differing judgments as target
uncertainty \citep{pmlr-v130-mimori21a,riera2026multirater}. Our model presupposes
one meaningful binary outcome, so disagreement must be interpretable as
imperfect measurement rather than equally valid alternative labels.

\subsection{Scope}
\label{app:scope}
Known marginal-channel correction permits annotation dependence;
identification requires the assumptions in Section~\ref{v03:identification}.
Adding annotators does not establish conditional independence.
ECE-estimation accuracy, small recalibrated ECE, and classification accuracy are distinct. Corrected-loss guarantees concern
clean risk, histogram guarantees concern calibration, and IV optimality concerns
known-channel variance under its stated conditions.
Human-label results stress-test the method without verifying its noise assumptions or clinical validity. Benchmark labels serve as evaluation references; shared evaluation samples limit interpretation of pilot-fit SDs.

\end{document}